\documentclass[a4paper,fleqn]{cas-dc}
\pdfoutput=1

\usepackage[authoryear,longnamesfirst]{natbib}
\usepackage{amsmath}
\usepackage{amssymb}
\usepackage{amsthm}
\usepackage{bm}
\usepackage{graphicx}
 \graphicspath{{figures/}} 
\usepackage{algorithm}
\usepackage{algorithmic}
\usepackage{subfigure}

\newtheorem{theorem}{Theorem}
\newtheorem{lemma}{Lemma}
\newtheorem{corollary}{Corollary}
\newtheorem{definition}{Definition}
\newtheorem{condition}{Condition}

\def\tsc#1{\csdef{#1}{\textsc{\lowercase{#1}}\xspace}}
\tsc{WGM}
\tsc{QE}

\begin{document}
\let\WriteBookmarks\relax
\def\floatpagepagefraction{1}
\def\textpagefraction{.001}

\shorttitle{Generalized-activated Deep Double Deterministic Policy Gradients}    

\shortauthors{Jiafei Lyu, et al.}  

\title[mode = title]{Value Activation for Bias Alleviation: Generalized-activated Deep Double Deterministic Policy Gradients}  

\tnotemark[1] 
%
\tnotetext[1]{This research was partly supported by the National Natural Science Foundation of China (Grant No. 41876098), the National Key Research and Development Program of China (Grant No. 2020AAA0108303), and Shenzhen Science and Technology Project (Grant No. JCYJ20200109143041798).}
%

%

\author[1]{Jiafei Lyu}[orcid=0000-0001-6616-417X]



\ead{lvjf20@mails.tsinghua.edu.cn}



\affiliation[1]{organization={Tsinghua Shenzhen International Graduate School, Tsinghua University},
            addressline={Lishui Road, Nanshan District}, 
            city={Shenzhen},
            postcode={518055}, 
            country={China}}

\author[1]{Yu Yang}[]



\ead{yy20@mails.tsinghua.edu.cn}




\author[2]{Jiangpeng Yan}[]

\fnmark[$^*$]

\ead{yanjp17@mails.tsinghua.edu.cn}



\affiliation[2]{organization={Department of Automation, Tsinghua University},
            addressline={Shuangqing Road, Haidian District}, 
            city={Beijing},
            postcode={100084}, 
            country={China}}

\author[1]{Xiu Li}[]

\fnmark[$^*$]

\ead{li.xiu@sz.tsinghua.edu.cn}

\cortext[1]{Corresponding author}



\begin{abstract}
It is vital to accurately estimate the value function in Deep Reinforcement Learning (DRL) such that the agent could execute proper actions instead of suboptimal ones. However, existing actor-critic methods suffer more or less from underestimation bias or overestimation bias, which negatively affect their performance. In this paper, we reveal a simple but effective principle: \emph{proper value correction benefits bias alleviation}, where we propose the generalized-activated weighting operator that uses \emph{any} non-decreasing function, namely activation function, as weights for better value estimation. Particularly, we integrate the generalized-activated weighting operator into value estimation and introduce a novel algorithm, Generalized-activated Deep Double Deterministic Policy Gradients (GD3). We theoretically show that GD3 is capable of alleviating the potential estimation bias. We interestingly find that simple activation functions lead to satisfying performance with no additional tricks, and could contribute to faster convergence. Experimental results on numerous challenging continuous control tasks show that GD3 with task-specific activation outperforms the common baseline methods. We also uncover a fact that fine-tuning the polynomial activation function achieves superior results on most of the tasks. Codes will be available upon publication.
\end{abstract}


\begin{highlights}
\item We propose a novel generalized-activated weighting operator for bias alleviation in deep reinforcement learning. 
\item We theoretically and experimentally show that the distance between the max operator and the generalized-activated weighting operator can be bounded.
\item We show theoretically and experimentally that generalized-activated weighting operator helps alleviate both underestimation bias and overestimation bias.
\item We find that simple activation functions are enough for amazing performance without any tricks and special design for activation function. 
\end{highlights}


\maketitle

\section{Introduction}
Reinforcement Learning (RL) has achieved remarkable progress in continuous control scenarios in recent years (\cite{mnih2015human}). Actor-critic methods (\cite{konda2000actor}) are widely adopted in continuous control problems, which relies on the critic to estimate the value function under the current state and execute policy via an actor. In actor-critic-style methods, it is important for the critic to estimate the value function in a good manner such that the actor can be driven to learn better policy. However, existing methods are far from satisfying. 

As an extension of Deterministic Policy Gradient (DPG) (\cite{silver2014deterministic}), Deep Deterministic Policy Gradient (DDPG) (\cite{lillicrap2015continuous}) is one of the earliest and most widely-used algorithms. DDPG is designed to operate on a continuous regime with a deterministic policy based on the Actor-Critic method. Although DDPG performs well in many continuous control tasks, it actually suffers from overestimation bias  (\cite{fujimoto2018addressing}), as is often reported in Deep Q-Network (DQN) (\cite{van2016deep,zhang2017weighted}). Such phenomenon is caused by the fact that the actor network in DDPG learns to take action according to the maximal value estimated by the critic network, which would result in poor performance. To address the function approximation error in DDPG, Fujimoto et al. propose Twin Delayed Deep Deterministic Policy Gradient (TD3) algorithm (\cite{fujimoto2018addressing}) by training two critic networks and updating with the smaller estimated value to tackle the problem, which is inspired to the idea of Double DQN (\cite{van2016deep}). TD3 significantly outperforms DDPG with reduced overestimation bias with Clipped Double Q-learning and delayed update of target networks. While it turns out that TD3 may induce large underestimation bias (\cite{ciosek2019better}).

In this paper, we reveal a simple but effective principle: \emph{proper value correction benefits bias alleviation}, i.e., the bias in value estimation can be eased. Particularly, we propose a novel and general function transformation framework for value estimation bias alleviation in deep reinforcement learning. Our method leverages \emph{any} non-decreasing function, namely activation function, and leads to a novel \emph{generalized-activated weighting operator}. The operator is an \emph{activation} to the estimated Q values which generates a new distribution with a one-to-one mapping to the original value distribution as shown in Fig \ref{fig:gd3fig}. The activated distribution reveals the importance of the estimated value in the distribution. Intuitively, we can get a better value estimation by weighting these estimated values. We show that the generalized-activated weighting operator helps alleviate overestimation bias in DDPG if built upon single critic network, where we develop Generalized-activated Deep Deterministic Policy Gradient (GD2) algorithm. It also helps relieve underestimation bias in TD3 if built upon double critics, which incurs \underline{G}eneralized-activated \underline{D}eep \underline{D}ouble \underline{D}eterministic Policy Gradients (GD3) algorithm. Extensive experimental results on challenging continuous control tasks in OpenAI Gym (\cite{brockman2016openai}) show that GD3 significantly outperforms common baseline methods and could achieve better performance and sample efficiency with task-specific activations.

\begin{figure}
    \centering
    \includegraphics[scale=0.35]{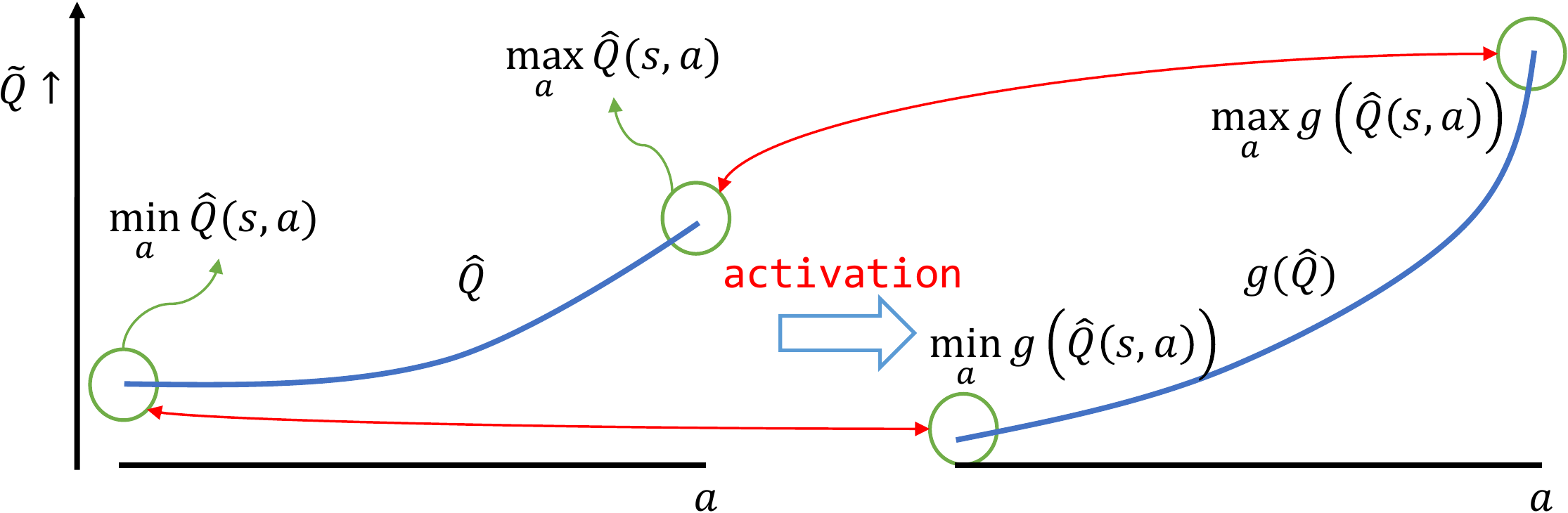}
    \caption{How activation function changes the distribution of value function over action space. Large Q would become larger after being activated.}
    \label{fig:gd3fig}
\end{figure}

Our contributions can be summarized as follows:
\begin{itemize}
    \item We propose a novel generalized-activated weighting operator for bias alleviation in deep reinforcement learning. 
    \item We theoretically and experimentally show that the distance between the max operator and the generalized-activated weighting operator can be bounded.
    \item We show theoretically and experimentally that generalized-activated weighting operator helps alleviate both underestimation bias and overestimation bias.
    \item We find that simple activation functions are enough for amazing performance without any tricks and special design for activation function. 
\end{itemize}

The rest of the paper is organized as follows: Section \ref{sec:relatedwork} is a collection of related work; Section \ref{sec:pre} gives preliminaries of reinforcement learning; Section \ref{sec:gdw} formally defines generalized-activated weighting operator and gives theoretical analysis; Section \ref{sec:easebias} develops GD2 and GD3 algorithm; Section \ref{sec:experiment} presents detailed experimental results of GD3; Section \ref{sec:conclusion} concludes the paper.

\section{Related Work}
\label{sec:relatedwork}
Reinforcement learning has witnessed impressive progress in recent years (\cite{mnih2015human, mastering2015, Arulkumaran2019AlphaStarAE, Mirhoseini2020ChipPW}), and has been applied into many fields, like games (\cite{Ye2020TowardsPF, Ye2020MasteringCC, Liu2021SelfplayRL}), complex system control (\cite{Chen2021FaulttolerantAT, Huang2020AdaptivePS, Plos2020ComparisonOQ}), financial trading (\cite{Ma2021APM}), recommendation (\cite{Liu2020TopawareRL, Pei2019ValueawareRB}), etc.

Actor-Critic methods (\cite{konda2000actor}) are widely used in reinforcement learning which involve value function approximation for the estimation of policy gradient (\cite{sutton2000policy}). Many improvements and advances upon actor-critic methods are proposed in recent years, including asynchronous method (\cite{mnih2016asynchronous}), maximum entropy reinforcement learning (\cite{haarnoja2018soft}), distributional RL \cite{bellemare2017distributional, Ma2020DSACDS}, offline RL (batch RL) \cite{Levine2020OfflineRL, Yang2021BelieveWY, Jiang2021OfflineDM}, etc. DDPG (\cite{lillicrap2015continuous}), which is built upon actor-critic and sheds lights to continuous control, extends DPG (\cite{silver2014deterministic}) algorithm to deep reinforcement learning. There are many works that are built on and improve DDPG, such as model-based (\cite{gu2016continuous}), distributed (\cite{popov2017data}), distributional (\cite{barth2018distributed,bellemare2017distributional}), prioritized experience replay (\cite{horgan2018distributed}) and multi-step returns (\cite{meng2020effect}). While DDPG may cause devastating overestimation bias (\cite{fujimoto2018addressing}) which is also often reported in DQN and is relieved by Double DQN in discrete regime (\cite{van2016deep}).

In this paper, we are interested in ways of getting a better value estimator for continuous control tasks. How to obtain satisfying value estimation remains an important problem to be solved in reinforcement learning, and has been widely studied (\cite{pan2020softmax, lyu2021efficient, Kuznetsov2021AutomatingCO}). TD3 (\cite{fujimoto2018addressing}) improves value estimation of DDPG with Clipped Double Q-learning while it can lead to large underestimation bias (\cite{ciosek2019better}). Kuznetsov et al. adopt a truncated ensemble of distributional critics for controlling overestimation in continuous setting (\cite{kuznetsov2020control}). Moskovitz et al. balance between optimism and pessimism with a multi-arm bandit aiming at maximizing the immediate performance improvement (\cite{Moskovitz2021tactical}). Weighted Delayed DDPG (\cite{He2020ReducingEB}) adopts weighted sum of the minimum of Q-functions and the average of an ensemble of Q-functions for overestimation bias alleviation. A recent work (\cite{Cetin2021LearningPF}) studies a flexible data-driven way of bias controlling method with the aid of epistemic uncertainty. Different from these advances, we focus on improving the value estimation by leveraging \emph{any} non-decreasing activation functions. Our method does not involve training critic ensembles, estimating uncertainty or bandit components, which makes it much lighter and efficient.

Most relevant to our work is SD3 (\cite{pan2020softmax}), which uses the softmax operator on value function to get a softer value estimation. It is worth noting that SD3 is a special case of our method, i.e., we set the activation function as exponential and then GD3 would degenerate into SD3. Also, according to our experimental results, exponential activation is often not the optimal activation function (see Section \ref{sec:experiment}). GD3 is a general and effective method for bias alleviation in continuous control.

\section{Preliminaries}
\label{sec:pre}
Reinforcement learning aims to maximize expected discounted cumulative future reward and it can be formulated by a Markov Decision Process (MDP) which is made up of state space $\mathcal{S}$, action space $\mathcal{A}$, reward function $r: \mathcal{S} \times \mathcal{A} \rightarrow \mathbb{R}$, transition probability distribution $p(s^\prime|s,a)$, and discount factor $\gamma \in [0,1]$. In this paper, we consider continuous action space and reward function scenarios and assume they are bounded as required in (\cite{pan2020softmax,silver2014deterministic}), i.e., $r(s,a)\in[-R_{\mathrm{max}}, R_{\mathrm{max}}], \forall s,a$. DDPG is an off-policy algorithm for continuous control which learns deterministic policy $\pi(\cdot;\phi)$ instead of stochastic policy as in Q learning (\cite{watkins1989learning}). It is an extension of Deterministic Policy Gradient (DPG) (\cite{silver2014deterministic}) with deep neural networks and actor-critic structure. Since it is expensive to take max operator on continuous action space directly, the critic of DDPG estimates Q value by maximizing the Q-function $Q$: $\mathcal{S}\times\mathcal{A}\rightarrow\mathbb{R}$ to approximate max operator and is updated by minimizing $\mathbb{E}[(y_t-Q(s_t, a_t;\theta))^2]$, where $y_t=r_t +\gamma Q(s_{t+1}, \pi (s_{t+1}))$. $r_t$ is the scalar reward at time step $t$ and $Q(s_t, a_t;\theta)$ parameterized by $\theta$ is estimated by neural networks to approximate the true value $\theta^{\mathrm{true}}$; The actor decides what action to take and is updated via deterministic policy gradient:
\begin{equation}
    \nabla_\phi J(\pi(\cdot;\phi)) = \mathbb{E} [\nabla_\phi (\pi(s;\phi)) \nabla_a Q(s,a; \theta) |_{a=\pi(s,\phi)}].
\end{equation}

\section{Generalized-activated Weighting Operator}
\label{sec:gdw}

\subsection{How to better estimate value function?}

In Actor-Critic, it is crucial for the critic network to effectively estimate the value function as the performance of the actor network is hugely influenced by its counterpart critic. TD3 alleviates the overestimation bias in DDPG (\cite{fujimoto2018addressing}) while it may result in large underestimation bias (\cite{ciosek2019better}). Both of the biases would negatively affect the performance of the algorithm. We, however, propose to do value estimate correction using \emph{any} non-decreasing functions for more balanced value estimation. With function transformation, a new distribution for Q-values can be acquired. We name the transformation from one distribution to another as {\it activation}, and the corresponding non-decreasing function as {\it activation function}.

Fig~\ref{fig:gd3fig} shows the process of activation from original value function distribution to activated one via a non-decreasing function $g(\cdot;\bm{\psi})$ parameterized by $\bm{\psi}$ where $\bm{\psi}$ can be a scalar or a vector. Note that we suppose that one action corresponds to one unique value function here for simplicity and better illustration which may not be necessarily true in the continuous regime. We require the activation function to be non-decreasing as we need larger weight upon larger Q value to prevent devastating underestimation bias. The activation process is actually a one-to-one order-preserving mapping (as the red line in Fig~\ref{fig:gd3fig} shows) from one probability space to another probability space with larger variance, i.e., 
the diversity of the value distribution would become larger after being activated. 

The operator $\mathcal{T}$ is a value estimator which gives estimation of value function given state $s$ and is used to update target value, i.e. $y = r+\gamma \mathcal{T}(s)$. Then we could give a formal definition of the generalized-activated weighting operator.

\begin{definition}
 Let $g(\cdot;\bm{\psi})$ be a non-decreasing function parameterized by $\bm{\psi}$, then the generalized-activated weighting operator is defined as: 
 \begin{equation}
     GA_{g}(Q(s,\pi(s));\bm{\psi}) = \int_{a\in\mathcal{A}}\frac{g(Q(s,a);\bm{\psi})Q(s,a)}{\int_{a^\prime\in\mathcal{A}}g(Q(s,a^\prime);\bm{\psi})da^\prime}da.
 \end{equation}
\end{definition}

\subsection{Theoretical Analysis}
Theorem \ref{theo:ga} and Theorem \ref{theo:value} guarantee the rationality and practicability of generalized-activated weighting operator. Note that some proofs of theorems in this part and section \ref{sec:easebias} are deferred to Appendix A.

We first give a simple lemma under the bounded reward function assumption.

\begin{lemma}
\label{lemma:qbound}
The Q value is bounded, i.e., $-\frac{R_{\mathrm{max}}}{1-\gamma}\le Q(s,a)\le\frac{R_{\mathrm{max}}}{1-\gamma}, \forall s,a$.
\end{lemma}

\begin{proof}
By definition, we have 
\begin{equation*}
\begin{aligned}
Q(s,a) &= \mathbb{E}_{\pi}\left[\sum_{t=1}^T\gamma^t r(s_t,a_t) \bigg | a\sim\pi(\cdot|s)\right] \\
&\le \mathbb{E}_{\pi} \left[ R_{\mathrm{max}}\sum_{t=1}^T \gamma^t \right] = \frac{R_{\mathrm{max}}}{1-\gamma}.
\end{aligned}
\end{equation*}

Similarly, we have 
\begin{equation*}
\begin{aligned}
Q(s,a) &= \mathbb{E}_{\pi}\left[\sum_{t=1}^T\gamma^t r(s_t,a_t)\bigg|a\sim\pi(\cdot|s)\right] \\
 &\ge \mathbb{E}_{\pi} \left[- R_{\mathrm{max}}\sum_{t=1}^T \gamma^t \right] = -\frac{R_{\mathrm{max}}}{1-\gamma}.
\end{aligned}
\end{equation*}
\end{proof}

Our first main result reveals that the distance between generalized-activated weighting operator and the max operator can be bounded.

\begin{theorem}
\label{theo:ga}
Denote $\mathcal{C}(Q,s,\epsilon) = \{a|a\in\mathcal{A}, Q(s,a)\ge \max_a Q(s,a) - \epsilon\}$, $F(Q,s,\epsilon) = \int_{a\in\mathcal{C}(Q,s,\epsilon)}1da$ for $\epsilon>0$. Let $\beta\in\mathcal{B}$ such that $g(Q(s,a);\bm{\psi}) \ge \exp(\beta Q(s,a))$, and denote $T(Q^*)$ as the maximum value of the function $T(Q(s,a);\bm{\psi},\beta) = \frac{1}{\beta}\ln g(Q(s,a);\bm{\psi}) - Q(s,a)$ for some non-decreasing functions $g(Q(s,a);\bm{\psi})$, then the difference between the max operator and generalized-activated weighting operator is bounded:
$$
0\le \max_aQ(s,a) - GA_{g}(Q(s,\pi(s));\bm{\psi}) \le M(Q,\epsilon;\bm{\psi},\beta),
$$
where $M(Q,\epsilon;\bm{\psi},\beta) = \epsilon + \frac{\int_{a\in\mathcal{A}}1da - 1 + \ln F(Q,s,\epsilon)}{\beta} + T(Q^*)$.
\end{theorem}
\begin{proof}
See Appendix A.1.
\end{proof}

\noindent{\bf Remark 1:} If we adopt exponential function as activation function, then (1) $T(Q^*)$ would become a constant; (2) the generalized-activated weighting operator would degenerate to softmax operator, which is recently studied in (\cite{pan2020softmax}). 

\noindent{\bf Remark 2:} The upper bound would be 0 if $Q(s,a)$ is a constant function since $$
\begin{aligned}
GA_g(Q(s,\pi(s));\psi) &= \int_{a\in\mathcal{A}}\frac{g(Q(s,a);\psi)Q(s,a)}{\int_{a^\prime\in\mathcal{A}}g(Q(s,a^\prime);\psi)da^\prime}da \\
&= Q(s,a) = \max_a Q(s,a).
\end{aligned}
$$ The upper bound would converge to $\epsilon+T(Q^*)$ with $\beta\rightarrow\infty$. 

$T(Q^*)$ denotes the cost of the generalization from exponential family to arbitrary non-decreasing functions which measures the maximal distance between an arbitrary non-decreasing function to exponential activation parameterized by $\beta$. If $T(Q^*)$ always lies in a reasonable region, then the upper bound would be $O(\frac{1}{\beta})$. In fact, the distance between max operator and generalized-activated weighting operator is reasonable enough in practice where detailed discussion in available in section \ref{sec:activatecomp}. 

Lemma~\ref{lemma:set} says that we can always find such non-decreasing functions satisfying the inequality $g(Q(s,a);\bm{\psi})\ge\exp(\beta Q(s,a))$, which ensures the rationality of the theorem.

\begin{lemma}
\label{lemma:set}
The support set $\mathcal{B}$ containing $\beta$ such that for some non-decreasing functions $g(\cdot;\psi)$, $g(Q;\psi)\ge\exp (\beta Q)$, is non-empty, i.e., $| \mathcal{B} | >$ 0.
\end{lemma}

\begin{proof}
Note that when $\beta\rightarrow 0$, i.e., $\beta$ is sufficiently small, the exponential operator would become similar to that of linear operator, i.e., $\exp (\beta Q(s,a)) \approx 1+\beta Q$ if $|\beta|<\epsilon$. Hence we can always design some non-decreasing functions with parameter $\psi$ such that $g(Q(s,a);\psi) \ge \exp(\beta Q(s,a))$, e.g., $g(Q(s,a);\psi) = 1 + 2 Q(s,a)$ as $2\gg\beta$. Furthermore, the $Q$-value is bounded based on Lemma \ref{lemma:qbound}, therefore as long as $\beta$ is small enough, we can always find some non-decreasing functions that satisfy the constraint. Hence $|\mathcal{B}|>0$.
\end{proof}

After applying generalized-activated weighting operator, the value function would be updated by $Q_{t+1} = r_t + \gamma\mathbb{E}_{s^\prime\sim p(\cdot|s,a)}[GA_{g}(Q(s,a);\bm{\psi})]$, then by using Theorem \ref{theo:ga}, the distance of value functions induced by optimal operator and generalized-activate weighting operator is also bounded.

\begin{theorem}[Value Iteration]
\label{theo:value}
The distance between the optimal value function $V^*$ and the value function weighted by generalized-activated weighting operator is bounded at t-th iteration:
$$
\|V_t-V^*\|_\infty \le \gamma^t \|V_0(s) - V^*(s)\|_\infty + N(Q, \epsilon, \gamma;\psi,\beta),
$$
where $N(Q, \epsilon, \gamma;\psi,\beta) = \frac{T(Q^*)}{1-\gamma} +  \frac{1}{1-\gamma}\frac{\beta\epsilon+\int_{a\in\mathcal{A}}1da-1}{\beta}-\sum_{k=1}^t \gamma^{t-k}\frac{\min_s \ln F(Q_k(s,\pi(s)),s,\epsilon)}{\beta}$.
\end{theorem}

\begin{proof}
Note that the optimal value estimation function $V^*$ is acquired by the max operator, i.e.,
\begin{equation}
    \label{eq:maxoptimal}
    V^*(s)=\max_{a\in\mathcal{A}} Q^*(s,a).
\end{equation}

By definition of generalized-activated weighting operator value iteration, we have
\begin{equation*}
    \label{eq:valexp}
    \begin{aligned}
    |V_{t+1} - V^*| &= |GA_g(Q(s,\pi(s));\psi) - \max_a Q^*(s,a)| \\
    &\le |GA_g(Q(s,\pi(s));\psi) - \max_a Q_{t+1}(s,a)| \\
    &\quad + |\max_a Q_{t+1}(s,a) - \max_a Q^*(s,a)|.
    \end{aligned}
\end{equation*}

By using Theorem \ref{theo:ga} and the fact that max operator is non-expansive (\cite{singh1994convergence}), we have $|\max_a Q_{t+1}(s,a) - \max_a Q^*(s,a)| = \max_a|Q_{t+1}(s,a) - Q^*(s,a)| \le \gamma \max_{s^\prime} |V_t(s^\prime) - V^*(s^\prime)|$. Then,
\begin{equation*}
    \label{eq:the2iter}
    \begin{aligned}
    \| V_{t+1}(s) - V^*(s) \|_\infty &\le \gamma \| V_t(s) - V^*(s) \|_\infty +\epsilon+T(Q^*)\\
    & \, + \frac{\int_{a\in\mathcal{A}}1da - 1 - \min_s \ln F(Q_{t+1},s,\epsilon)}{\beta}.
    \end{aligned}
\end{equation*}

Doing iteration, we have
\begin{equation*}
    \label{eq:the2con}
    \begin{aligned}
    &\| V_t(s) - V^*(s) \|_\infty \\ &\le \gamma^t \| V_0(s) - V^*(s) \|_\infty +\sum_{k=1}^t \gamma^{t-k}(\epsilon+T(Q^*)) \\ & \qquad- \sum_{k=1}^t \frac{\min_s \ln F(Q_{t+1},s,\epsilon)}{\beta} + \sum_{k=1}^t \frac{\int_{a\in\mathcal{A}}1da - 1}{\beta} \\
    & \le \gamma^t \| V_0(s) - V^*(s) \|_\infty + \frac{\epsilon+T(Q^*)}{1-\gamma} + \frac{\int_{a\in\mathcal{A}}1da - 1}{(1-\gamma)\beta} \\ & \qquad- \sum_{k=1}^t \frac{\min_s \ln F(Q_{t+1},s,\epsilon)}{\beta} \\
    &= \gamma^t \| V_0(s) - V^*(s) \|_\infty + N(Q, \epsilon, \gamma;\psi,\beta).
    \end{aligned}
\end{equation*}
\end{proof}

\noindent{\bf Remark 3:} The error bound would converge to $\frac{T(Q^*)+\epsilon}{1-\gamma}$ with iteration.

\section{Bias alleviation with Generalized-activated Deep Double Deterministic Policy Gradients}
\label{sec:easebias}

Since generalized-activated weighting operator would deviate the maximum value and lead to a softer value estimation, then will it help achieve an intermediate status between overestimation bias and underestimation bias? 

\subsection{Generalized-activated weighting operator upon single critic}
We start from applying the generalized-activated weighting operator on the value estimation function of DDPG, and we have GD2 algorithm (see Algorithm \ref{alg:alggd2} for full algorithm of GD2) which is a variant but efficient algorithm as guaranteed by the following Corollary:

\begin{corollary}
\label{coro:gd2}
 Suppose that the actor is a local maximizer with respect to the critic, then there exists noise clipping parameter $c > 0$ such that the bias of GD2 is smaller than that of DDPG, i.e. bias($\mathcal{T}_{GD2}$) $\le$ bias($\mathcal{T}_{DDPG}$).
\end{corollary}

\begin{proof}
The bias caused by an operator $\mathcal{T}$ is: 
$
    \mathbb{E}[\mathcal{T}(s) - \mathbb{E}[Q(s,\pi(s;\phi^{-});\theta^{\prime})]]
$ where $\theta^\prime$ is the true parameter of value function. Then it is easy for us to find that $bias(\mathcal{T}_{GD2}) - bias(\mathcal{T}_{DDPG}) = \mathbb{E}[\mathcal{T}_{GD2}(s)] - \mathbb{E}[\mathcal{T}_{DDPG}(s)]$. By definition, we have
\begin{equation}
    \label{eq:oper}
    \begin{aligned}
    &\mathcal{T}_{DDPG}(s) = Q(s,\pi(s;\phi^{-});\theta^{-}), \\
    &\mathcal{T}_{GD2}(s) = GA_g(Q(s,\pi(s);\theta^-);\psi).
    \end{aligned}
\end{equation}

The actor is the local maximizer with respect to the critic, then there exist some noise clipping parameter such that for any state $s$, the selected action by the policy would induce to a local maximum, that is $Q(s,\pi(s;\phi^{-});\theta^{-}) = \max_a Q(s,a;\theta^-)$.

By using the left-hand-side inequality of Theorem 1, we have
\begin{equation}
    \label{eq:gd2}
    \begin{aligned}
    &\quad bias(\mathcal{T}_{GD2}) - bias(\mathcal{T}_{DDPG}) \\
    &= \mathbb{E}[GA_g(Q(s,\pi(s);\theta^-);\psi)] - \mathbb{E}[\max_a Q(s,a;\theta^-)] \\
    &\le \mathbb{E}[\max_a Q(s,a;\theta^-) - \max_a Q(s,a;\theta^-)] = 0.
    \end{aligned}
\end{equation}

Hence, $bias(\mathcal{T}_{GD2})\le bias(\mathcal{T}_{DDPG})$.

\end{proof}

Corollary~\ref{coro:gd2} reveals an important fact that we could ease overestimation bias with {\it any} non-decreasing activation function as weights compared with the max operator, i.e., generalized-activated weighted operator helps alleviate overestimation bias.

\begin{algorithm}[tb]
\caption{Generalized-activated Deep Deterministic Policy Gradient (GD2)}
\label{alg:alggd2}
\begin{algorithmic}[1] 
\STATE Initialize critic networks $Q$ and actor networks $\pi$ with random parameters $\theta, \phi$; Initialize target network $\theta^\prime \leftarrow \theta, \phi^\prime \leftarrow \phi$. 
\STATE Initialize replay buffer $\mathcal{B} = \{\}$.
\STATE Given a non-decreasing function $g(\cdot;\bm{\psi})$ with parameter $\bm{\psi}$.
\FOR{$t$ = 1 to $T$}
\STATE Select action $a$ with Gaussian exploration noise $\epsilon$, $\epsilon\sim \mathcal{N}(0,\sigma)$ based on $\pi$ and observe reward $r$, new state $s^\prime$ and done flag $d$
\STATE Store transitions, i.e. $\mathcal{B}\leftarrow\mathcal{B}\bigcup \{(s,a,r,s^\prime,d)\}$
\STATE Sample $N$ transitions $\{(s_i,a_i,r_i,s_i^\prime,d_i)\}_{i=1}^N\sim\mathcal{B}$
\STATE $a^\prime\leftarrow \pi(s^\prime;\phi^\prime) + \epsilon$, $\epsilon\sim$ clip($\mathcal{N}(0,\bar{\sigma}),-c,c$)
\STATE Calculate $GA_g\left(Q(s^\prime,a^\prime); \bm{\psi} \right)$
\STATE $y_t \leftarrow r + \gamma(1-d) GA_g\left(Q(s^\prime,a^\prime);\bm{\psi} \right)$
\STATE Update critic network: \\ $\theta\leftarrow \arg\min_{\theta}\frac{1}{N}\sum_s (Q(s,a;\theta)-y)^2$
\STATE Update actor network: \\ $\nabla J_{\phi}(\phi) = \frac{1}{N}\sum_s \nabla_a Q(s,a;\theta_i)|_{a=\pi_{\phi}(s;\phi)}\nabla_{\phi}\pi(s;\phi)$
\STATE Update target networks: \\ $\theta^\prime \leftarrow \tau\theta + (1-\tau)\theta^\prime, \phi^\prime\leftarrow\tau\phi+(1-\tau)\phi^\prime$
\ENDFOR
\end{algorithmic}
\end{algorithm}

In order to illustrate the effectiveness of GD2 in overestimation bias reduction, we conduct experiments in one typical environment in MuJoCo (\cite{todorov2012mujoco}), Hopper-v2. We compare the performance of DDPG and GD2 with simple activations (e.g., polynomial activation and tanh activation). The polynomial activation has the following parameter setup: $k = 2, \alpha=0.05$ with bias $b=2$, and tanh activation has the parameter $\beta = 0.1$, $b=2$ (see Section \ref{sec:activatecomp} for detailed activation formula). We also record the value estimation bias of these two algorithms to investigate whether GD2 helps alleviate overestimation bias where the true values are evaluated by averaging the discounted long-term rewards acquired by rolling out the current policy that starts from the sampled states every $10^4$ steps.

\begin{figure*}
    \centering
    \subfigure[Estimation bias]{
    \label{fig:gd2hopperbias}
    \includegraphics[scale=0.5]{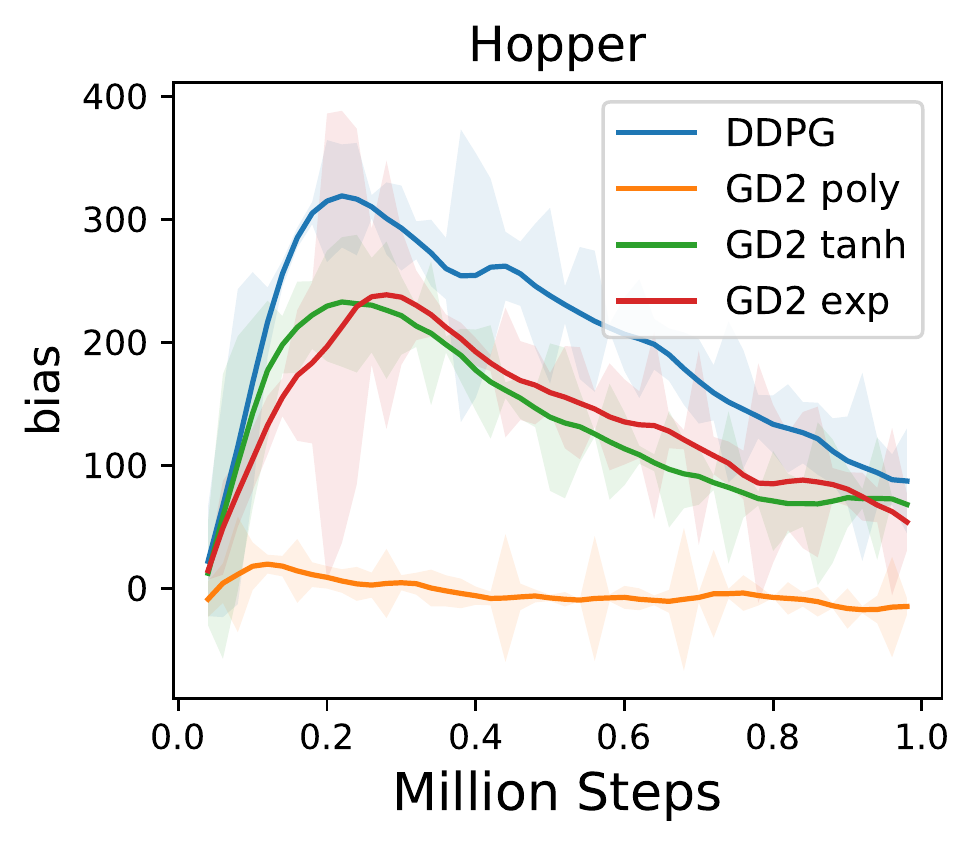}
    }\hspace{0mm}
    \subfigure[Performance]{
    \label{fig:gd2hopperperformance}
    \includegraphics[scale=0.5]{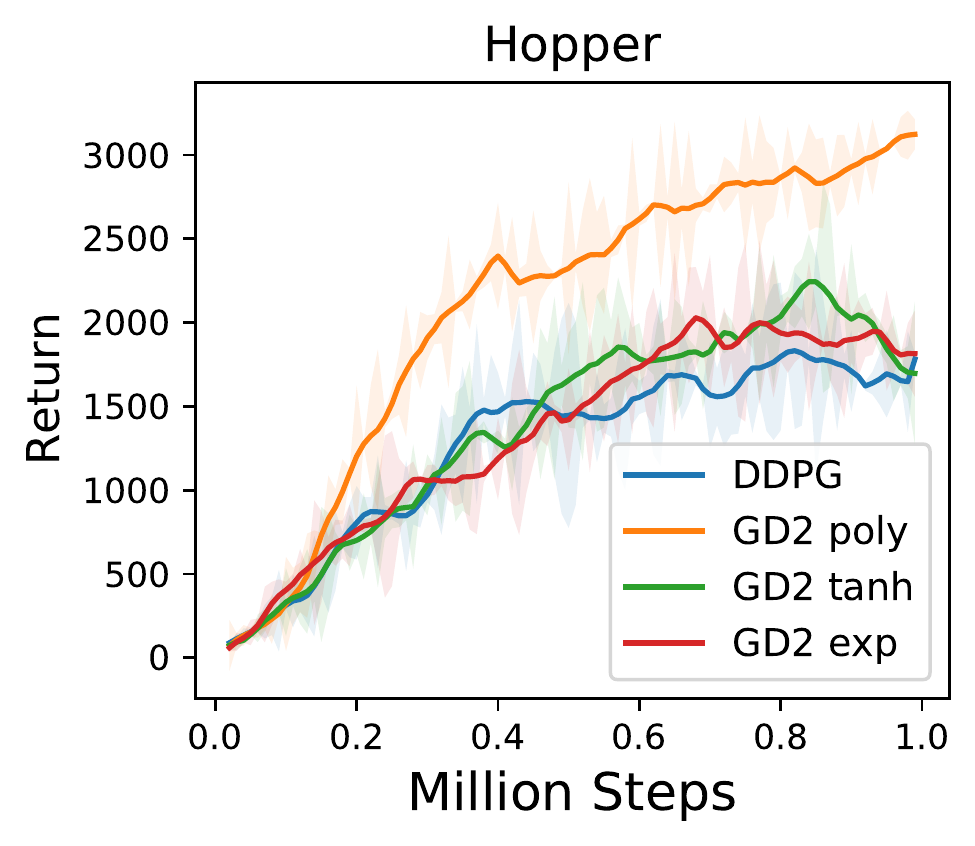}
    }\hspace{0mm}
    \caption{(a) Value estimation bias comparison, (b) Performance comparison between GD2 with different activations (polynomial activation, tanh activation, and exponential activation) and DDPG. The experiment is conducted on Hopper-v2 environment (5 runs, mean $\pm$ standard deviation). Note that exp denotes exponential activation, poly denotes polynomial activation, and tanh denotes tanh activation.}
    \label{fig:gd2hopper}
\end{figure*}

The experimental result is presented in Fig \ref{fig:gd2hopper}, where we observe significant estimation reduction compared to DDPG with the aid of generalized-activated weighting operator with different activation functions in Fig \ref{fig:gd2hopperbias}. GD2 significantly outperforms DDPG with activations like polynomial functions, or tanh functions as shown in Fig \ref{fig:gd2hopperperformance}. It is interesting to note that different activations have different performance and bias reduction effect on this task, e.g., polynomial activation has much less estimation bias and the best performance on Hopper-v2, which indicates that the optimal activation function is task-relevant or task-specific.

\subsection{Generalized-activated weighting operator upon double critics}
Generalized-activated weighting operator also helps with underestimation bias. Note that the direct application of generalized-activated weighting operator on the value estimation from two critics of TD3 would lead to larger underestimation bias because $\mathcal{T}_{GD2}(s^\prime) = GA_{g}(Q_i(s^\prime,a;\theta_i^\prime);\bm{\psi}) \le \max_a Q_i(s^\prime,a;\theta_i^\prime)$ based on Theorem \ref{theo:ga} where the right side is the value estimation by TD3. We therefore apply double actors and estimate the target value for critic $i$ by $y_i = r + \gamma \mathcal{T}_{GD3}(s^\prime)$ as in SD3 (\cite{pan2020softmax}), where $\mathcal{T}_{GD3} = GA_{g}(\hat{Q}_i(s^\prime,\pi(s^\prime));\bm{\psi})$ and $\hat{Q}_i(s^\prime,a) = \min_{i=1,2}\{Q_i(s^\prime,a;\theta_i^\prime)\}$. In this way, we can get an estimator with less bias than GD2 as guaranteed by Lemma \ref{lemma:gd2gd3}. Considering there exist integral in generalized-activated weighting operator, we rewrite $GA_{g}(\cdot)$ in terms of expectation with importance sampling to convert it into a manageable formula given by Eq. (\ref{eq:gd3target}) as suggested in (\cite{haarnoja2017reinforcement}).

\begin{equation}
    \label{eq:gd3target}
    \mathbb{E}_{a^\prime\sim p}\left[\frac{g(\hat{Q}(s^\prime,a^\prime;\bm{\psi})\hat{Q}(s^\prime,a^\prime))}{p(a^\prime)}\right]{\Big /}\mathbb{E}_{a^\prime\sim p}\left[\frac{g(\hat{Q}(s^\prime,a^\prime;\bm{\psi})}{p(a^\prime)}\right],
\end{equation}

where $p(a^\prime)$ is the Gaussian probability density function which depends on the clipped action values. The novel algorithm, \underline{G}eneralized-activated \underline{D}eep \underline{D}ouble \underline{D}eterministic Policy Gradients (GD3) is shown in Algorithm \ref{alg:alggd3}. 

A detailed computational graph description on the different updating process of TD3 and GD3 is available in Fig \ref{fig:td3gd3}. The general network architecture of GD3 is similar to that of TD3 while the update process is different. TD3 directly uses the minimal value estimation from two critic networks (Clipped Double Q-learning) while GD3 gives the value function an activation to create a new distribution that indicates the importance of each Q value and uses it as weights over action space to get a softer value estimation.

\begin{figure*}
    \centering
    \includegraphics[scale=0.4]{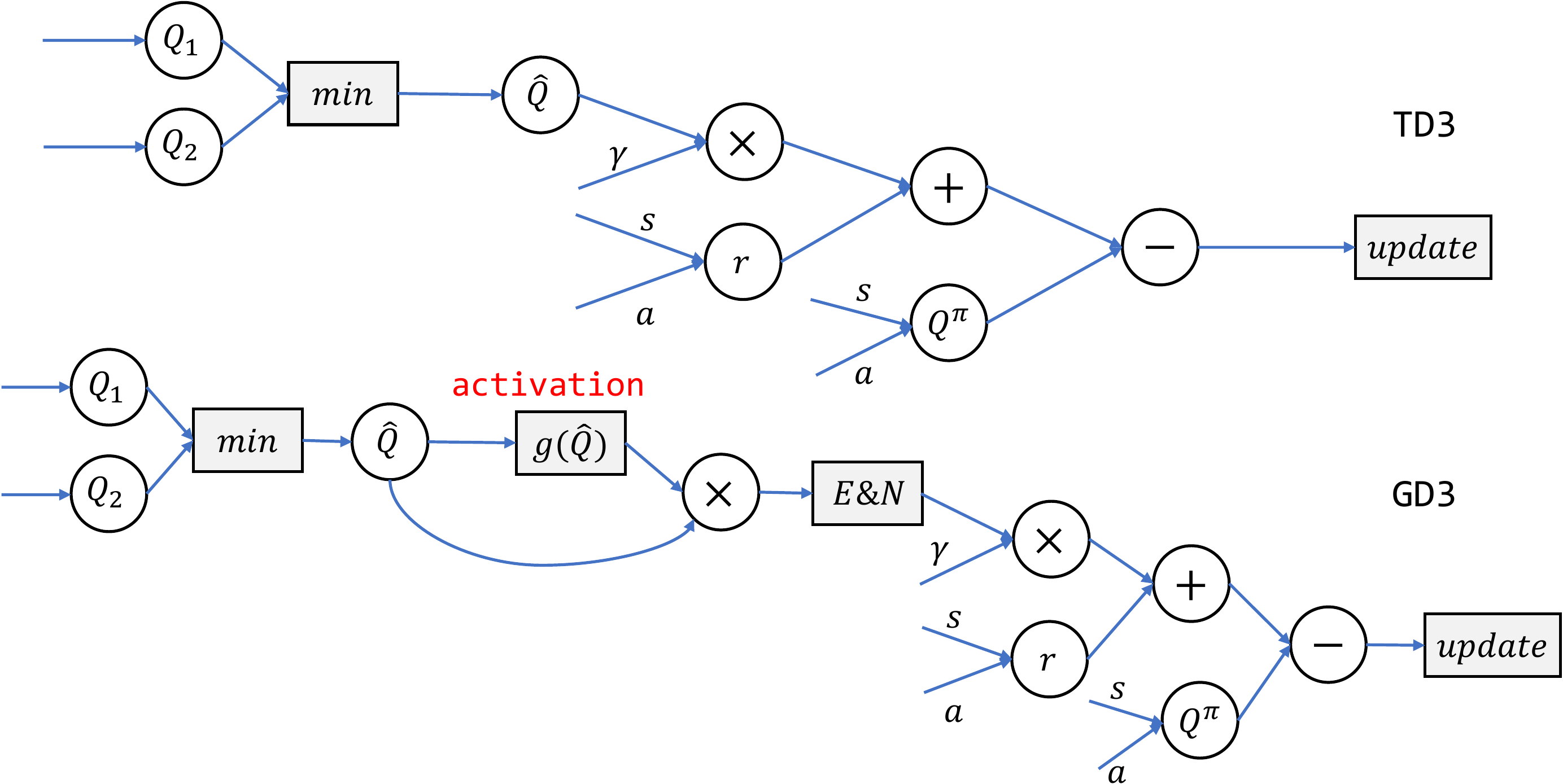}
    \caption{Graph describing the difference in updating between TD3 and GD3. E\&N denotes Expectation and Normalization.}
    \label{fig:td3gd3}
\end{figure*}

\begin{algorithm}[tb]
\caption{Generalized-activated Deep Double Deterministic Policy Gradients (GD3)}
\label{alg:alggd3}
\begin{algorithmic}[1] 
\STATE Initialize critic networks $Q_{\theta_1}, Q_{\theta_2}$ and actor networks $\pi_{\phi_1}, \pi_{\phi_2}$ with random parameters $\theta_1, \theta_2, \phi_1, \phi_2$; Initialize target network $\theta_1^\prime \leftarrow \theta_1, \theta_2^\prime \leftarrow \theta_2, \phi_1^\prime \leftarrow \phi_1, \phi_2^\prime \leftarrow \phi_2$. Initialize replay buffer $\mathcal{B} = \{\}$.
\STATE Given a non-decreasing function $g(\cdot;\bm{\psi})$ with parameter $\bm{\psi}$.
\FOR{$t$ = 1 to $T$}
\STATE Select action $a$ with Gaussian exploration noise $\epsilon$, $\epsilon\sim \mathcal{N}(0,\sigma)$ based on current policy and observe reward $r$, new state $s^\prime$ and done flag $d$
\STATE Store transitions, i.e. $\mathcal{B}\leftarrow\mathcal{B}\bigcup \{(s,a,r,s^\prime,d)\}$
\FOR{$i = 1,2$}
\STATE Sample $N$ transitions $\{(s_j,a_j,r_j,s_j^\prime,d_j)\}_{j=1}^N\sim\mathcal{B}$
\STATE $a^\prime\leftarrow \pi_{\phi_i}(s^\prime;\phi_i^\prime) + \epsilon$, $\epsilon\sim$ clip($\mathcal{N}(0,\bar{\sigma}),-c,c$)
\STATE $\hat{Q}(s^\prime,a^\prime)\leftarrow \min_{j=1,2}\left(Q_j(s^\prime, a^\prime; \theta_j^\prime) \right)$
\STATE Calculate $GA_g\left(\hat{Q}(s^\prime,a^\prime); \bm{\psi} \right)$ based on Eq.(\ref{eq:gd3target})
\STATE $y_i \leftarrow r + \gamma(1-d) GA_g\left(\hat{Q}(s^\prime,a^\prime);\bm{\psi} \right)$
\STATE Update critic networks: \\ $\theta_i\leftarrow \arg\min_{\theta_i}\frac{1}{N}\sum_s (Q_i(s,a;\theta_i)-y_i)^2$
\STATE Update actor networks: \\ $\nabla J_{\phi_i}(\phi_i) = \frac{1}{N}\sum_s \nabla_a Q_i(s,a;\theta_i)|_{a=\pi_{\phi_i}(s;\phi_i)}\nabla_{\phi_i}\pi(s;\phi_i)$
\STATE Update target networks: \\ $\theta_i^\prime \leftarrow \tau\theta_i + (1-\tau)\theta_i^\prime, \phi_i^\prime\leftarrow\tau\phi_i+(1-\tau)\phi_i^\prime$
\ENDFOR
\ENDFOR
\end{algorithmic}
\end{algorithm}

\begin{lemma}
 \label{lemma:gd2gd3}
 The bias of GD3 is smaller than that of GD2, i.e., bias($\mathcal{T}_{GD3}$) $\le$ bias($\mathcal{T}_{GD2}$).
\end{lemma}

\begin{proof}
 Suppose the activation function is identical for GD2 and GD3, and we denote it as $g(\cdot;\psi)$. We neglect the normalization term here as the integral of activated value function over action space gives similar results with identical activation function. There is no much difference between the value estimation from two critic networks. The bias of GD2 gives
 \begin{equation*}
     bias(\mathcal{T}_{GD2}) = \mathbb{E}\left[g(Q;\psi)Q(s,a^\prime) \right] - \mathbb{E}\left[Q(s,\pi(s;\phi_1^\prime);\theta^\prime) \right],
 \end{equation*}
 
 where $\theta^\prime$ is the true parameter. GD3 uses two critic networks and it takes the minimal value estimation, i.e.,
 \begin{equation*}
     \hat{Q}(s,\hat{a}) = \min_{i=1,2}Q_i(s,\hat{a};\theta_i^\prime).
 \end{equation*}
 
Obviously, $Q(s,a^\prime) \le \hat{Q}(s,\hat{a})$. Hence, the bias of GD3 gives
 \begin{equation*}
     bias(\mathcal{T}_{GD3}) = \mathbb{E}\left[g(\hat{Q};\psi)\hat{Q}(s,\hat{a}^\prime) \right] - \mathbb{E}\left[Q(s,\pi(s;\phi_i^\prime);\theta^\prime) \right].
 \end{equation*}
 
 Because the activation function $g(\cdot;\psi)$ is non-decreasing, it is easy to find $g(Q(s,a^\prime);\psi) \le g(\hat{Q}(s,\hat{a});\psi)$. Therefore, $g(Q;\psi)Q(s,a^\prime) \le g(\hat{Q};\psi)\hat{Q}(s,\hat{a}^\prime)$ which induce to the conclusion,
 \begin{equation*}
     bias(\mathcal{T}_{GD3}) \le bias(\mathcal{T}_{GD2}).
 \end{equation*}
\end{proof}

Lemma \ref{lemma:gd2gd3} paves the way for Theorem \ref{theo:gd3} which says that with proper activation function engineering satisfying Condition \ref{cond:ineq}, it is ensured that we could find a suitable activation function which could result in softer value estimation, i.e., it would alleviate both large overestimation bias and underestimation bias. Many function families meet the constraint, e.g. exponential family, linear family, etc.

\begin{condition}
\label{cond:ineq}
$\int_{a\in\mathcal{A}} g(Q(s,a);\psi) Q(s,a)da \int_{a^\prime\in\mathcal{A}}1da^\prime \ge \int_{a\in\mathcal{A}}g(Q(s,a);\psi)da\int_{a^\prime\in\mathcal{A}}Q(s,a^\prime)da^\prime.$
\end{condition}

\begin{theorem}[GD3 alleviates estimation bias]
\label{theo:gd3}
Suppose that the actor is a local maximizer with respect to the critic, then there exists noise clipping parameter $c > 0$ and some non-decreasing functions such that bias($\mathcal{T}_{TD3}$) $\le$ bias($\mathcal{T}_{GD3}$) $\le$ bias($\mathcal{T}_{DDPG}$).
\end{theorem}

\begin{proof}
See Appendix A.2.
\end{proof}

Note that an additional positive bias term $b>0$ could lead to reduced bias as stated in the following corollary:
\begin{corollary}
\label{coro:bias}
If some non-decreasing functions $g(\cdot;\bm{\psi})$ satisfy Condition~\ref{cond:ineq}, then a non-negative bias term would lead to a softer value estimation, i.e. the bias would be further reduced.
\end{corollary}

\begin{proof}
Suppose the non-decreasing function $g(\cdot;\psi)$ satisfies Condition \ref{cond:ineq}, let $b\ge 0$ be the bias term and denote $g^\prime(\cdot;\psi) = g(\cdot;\psi) + b$. Then,
\begin{equation*}
    \label{eq:biasreduce}
    \begin{aligned}
    & bias(\mathcal{T}_{GD3}(s;g)) - bias(\mathcal{T}_{GD3}(s;g^\prime))  \\
    & = \mathbb{E}\left[GA_g(\hat{Q}(s,\cdot);\psi) \right] - \mathbb{E}\left[GA_{g^\prime}(\hat{Q}(s,\cdot);\psi) \right] \\
    & = \frac{\int_{a\in\mathcal{A}}g(Q;\psi)Q(s,a)da}{\int_{a^\prime\in\mathcal{A}}g(Q(s,a^\prime);\psi)da^\prime} -
     \frac{\int_{a\in\mathcal{A}}(g(Q;\psi)+b)Q(s,a)da}{\int_{a^\prime\in\mathcal{A}}(g(Q(s,a^\prime);\psi)+b)da^\prime} \\
     &= \frac{b\left(\int_{a\in\mathcal{A}} g(Q(s,a);\psi) Q(s,a)da \int_{a^\prime\in\mathcal{A}}1da^\prime\right)}{\int_{a^\prime\in\mathcal{A}}g(Q(s,a^\prime);\psi)da^\prime\int_{a^\prime\in\mathcal{A}}(g(Q(s,a^\prime);\psi)+b)da^\prime} - \\
     & \qquad \frac{b\left(\int_{a\in\mathcal{A}}g(Q(s,a);\psi)da\int_{a^\prime\in\mathcal{A}}Q(s,a^\prime)da^\prime\right)}{\int_{a^\prime\in\mathcal{A}}g(Q(s,a^\prime);\psi)da^\prime\int_{a^\prime\in\mathcal{A}}\left[g(Q(s,a^\prime);\psi)+b\right]da^\prime}.
    \end{aligned}
\end{equation*}

Since $b\ge 0$ and by using Condition \ref{cond:ineq}, we get that
$$
bias(\mathcal{T}_{GD3}(s;g)) \ge bias(\mathcal{T}_{GD3}(s;g^\prime)).
$$
Thus, we get the conclusion that an additional non-negative bias term would lead to a softer value estimation.
\end{proof}

However, too large bias term is undesired based on the following corollary.
\begin{corollary}
\label{coro:zero}
The Generalized-activated weighting operator approximates mean operator with sufficiently large bias term, i.e., $\lim_{b\rightarrow +\infty}GA_g(Q(s,\pi(s));\psi,b)=\mathbb{E}_{a\in\mathcal{A}}[Q(s,a)]$.
\end{corollary}

\begin{proof}
By the definition of Generalized-activated weighting operator, if we add a bias term in the activation function, then
\begin{equation*}
     GA_{g}(Q(s,\pi(s));\bm{\psi},b) = \int_{a\in\mathcal{A}}\frac{(g(Q(s,a);\bm{\psi})+b)Q(s,a)}{\int_{a^\prime\in\mathcal{A}}(g(Q(s,a^\prime);\bm{\psi})+b)da^\prime}da.
 \end{equation*}
 
 With sufficiently large bias term, the activation function would degenerate to constant function. It is then easy to conclude that $\lim_{b\rightarrow +\infty}GA_g(Q(s,\pi(s));\psi,b)=\mathbb{E}_{a\in\mathcal{A}}[Q(s,a)]$.
\end{proof}

Actually, an additional bias term is essential as the activation could be generalized to be non-decreasing in $(0,+\infty)$ instead of $\mathbb{R}$ as long as we perform a translation transformation, i.e., add a constant value such that the value function is non-negative. Hence the constraint is further loosened with a suitable bias term. However, too large bias term would equally weigh Q values, which may induce unsatisfying results (see Section \ref{sec:ablation}). Furthermore, we experimentally find out that there is no need to strictly follow Condition \ref{cond:ineq} in practice as simple activations are adequate for a near-optimal estimation which is discussed in detail in Section \ref{sec:activatecomp}.

\section{Experiments}
\label{sec:experiment}
In this part, we conduct experiments on challenging MuJoCo (\cite{todorov2012mujoco}) and Box2d (\cite{catto2011box2d}) environments in OpenAI Gym (\cite{brockman2016openai}) with simple activations. GD3 shares network configuration with the default setting of TD3. For the polynomial activation function, we mainly consider the index term $k$ to be 2 or 3 for simplicity. The coefficient $\alpha$ is mainly chosen from $\{0.01, 0.05, 0.5, 0.1 \}$. For the tanh and exponential activation, the coefficient $\beta$ is mainly chosen from $\{0.005, 0.05, 0.1\}$. The bias term $b$ for all activations is chosen from $\{0, 1, 2, 5\}$ using grid search. All algorithms are run with 5 seeds and the performance is evaluated for 10 times every 5000 timesteps.

The network configuration for DDPG, TD3, and GD3 are similar, where we use fine-tuned DDPG instead of the vanilla DDPG as is suggested in (\cite{TD3}). The detailed hyperparameters for baseline algorithms and GD3 are shown in Table \ref{tab:hyperparameter}. We use the parameter suggested in (\cite{TD3}) for Humanoid-v2 where all algorithms fail if we use the same hyperparameters setup as other environments. We use the open-source implementation of SAC (Soft Actor Critic) in Github by the author (\href{https://github.com/haarnoja/sac}{https://github.com/haarnoja/sac}) where the hyperparameters are followed the same way as those in TD3 paper (\cite{fujimoto2018addressing}). The task-specific optimal activation parameter setup are shown in Table \ref{tab:activation}. Notice that the optimal activations are mostly polynomial, which indicates that fine-tuning the polynomial activation function can guarantee a satisfying performance on most of the tasks (see Fig \ref{fig:gd3result} for more detailed performance of polynomial activation). It also reveals that simple activation functions are adequate for satisfying performance.

\begin{table}
\centering
\begin{tabular}{lrr}
\toprule
\textbf{Hyperparameter}  & \textbf{Humanoid-v2} & \textbf{Other tasks} \\
\midrule
Shared & \\
\qquad Actor network  & $(256,256)$ & $(400,300)$ \\
\qquad Critic network & $(256,256)$ & $(400,300)$ \\
\qquad Batch size & $256$ & $100$ \\
\qquad Learning rate & $3\times 10^{-4}$ & $10^{-3}$ \\
\qquad Optimizer & \multicolumn{2}{c}{Adam} \\
\qquad Discount factor & \multicolumn{2}{c}{$0.99$} \\
\qquad Replay buffer size & \multicolumn{2}{c}{$10^6$}  \\
\qquad Warmup steps & \multicolumn{2}{c}{$10^4$} \\
\qquad Exploration noise & \multicolumn{2}{c}{$\mathcal{N}(0,0.1)$} \\
\qquad Noise clip & \multicolumn{2}{c}{$0.5$} \\
\qquad Target update rate & \multicolumn{2}{c}{$5\times 10^{-3}$} \\
\midrule
TD3  & \\
\qquad Target update interval & \multicolumn{2}{c}{$2$} \\
\qquad Target nosie & \multicolumn{2}{c}{$0.2$} \\
\midrule
GD3 & \\
\qquad Number of noises & \multicolumn{2}{c}{$50$} \\
\qquad Action noise variance & \multicolumn{2}{c}{$0.2$} \\
\bottomrule
\end{tabular}
\caption{Hyperparameters setup for baseline algorithms and GD3.}
\label{tab:hyperparameter}
\end{table}

\begin{table}
\centering
\begin{tabular}{lrrrr}
\toprule
\textbf{Environment}  & \textbf{activation} & \textbf{coef} & \textbf{index} & \textbf{bias} \\
\midrule
Ant-v2 & poly & $0.05$ & $3$ & $2$ \\ 
BipedalWalker-v3 & poly & $0.05$ & $2$ & $2$ \\
HalfCheetah-v2 & poly & $0.05$ & $3$ & $2$ \\
Hopper-v2 & poly & $0.05$ & $2$ & $2$ \\
Humanoid-v2 & poly & $0.05$ & $2$ & $2$ \\
LunarLancerContinuous-v2 & poly & $0.05$ & $3$ & $5$ \\
Swimmer-v2 & exp & --- & $500$ & $0$ \\
Walker2d-v2 & poly & $0.05$ & $2$ & $2$ \\
\bottomrule
\end{tabular}
\caption{Hyperparameters setup for activations of GD3 in different environments. Note that coef denotes coefficient, poly refers to polynomial activations, and exp refers to exponential activations. Polynomial activation has hyperparameters of coefficient $\alpha$, index term $k$, bias term $b$, and exponential activation has hyperparameters of index term $\beta$ and bias term $b$.}
\label{tab:activation}
\end{table}

\subsection{How different activation functions perform?}
\label{sec:activatecomp}

Both GD2 and GD3 stand for a huge family of methods with activated weighting of {\it any} non-decreasing functions, which may lead to activation function engineering dilemma, i.e., it may be confusing what activation function to choose when dealing with an unknown environment. We claim that there is no need for trivial design (it is hard to design functions that meet the constraint in Condition \ref{cond:ineq}), because we could get satisfying performance with some simple functions. We verify this by choosing three types of commonly used functions, polynomial function $g(x;\alpha,k,b) = \alpha x^k + b,\, (\alpha,k,b>0)$ and tanh $g(x;\beta,b) = \frac{\exp(\beta x)-\exp(-\beta x)}{\exp(\beta x) + \exp(-\beta x)}+b, \, (\beta, b>0)$, and exponential function $g(x;\beta,b) = \exp(\beta x) + b, (\beta, b>0)$, as activation function and testing the effectiveness of these activation functions under different parameters. The index term $k$ in polynomial activation is not required to be even as we would add a non-negative bias term. We select one typical environment in MuJoCo~\cite{todorov2012mujoco}, Hopper-v2, to uncover the properties and benefits of simple activations. The detailed hyperparameter setup is listed in Table \ref{tab:hyperparameter}.

\noindent {\textbf{Can operator distance be bounded?}} Based on Theorem~\ref{theo:ga}, the $T(Q^*)$ term in the upper bound remains uncertain with \textit{arbitrary} non-decreasing activations. Intuitively, if we assume $Q$ function is smooth, then $T(Q*)$ can be bounded as the activation function is non-decreasing and the $Q$ value is bounded. While it is essential to analyze whether the bound can be bounded with simple activation functions empirically. We run 5 independent experiments on Hopper-v2 with polynomial, tanh and exponential activations. It turns out that though poly and tanh deviate a little far away from max operator compared with exponential activation, the distance is actually measurable and it would gradually converge with the increment of training steps, which is shown clearly in Fig \ref{fig:hopperdistance} where {\it poly} refers to polynomial activation. The distance between max operator and generalized-activated weighting operator is doomed to be bounded as larger Q values always get larger weights which makes the weighted value estimation lean to the maximal value with value iteration thanks to the non-decreasing activation function. The convergence ensures that our estimator is valid.

\begin{figure*}
    \centering
    \subfigure[Distance]{
    \label{fig:hopperdistance}
    \includegraphics[scale=0.45]{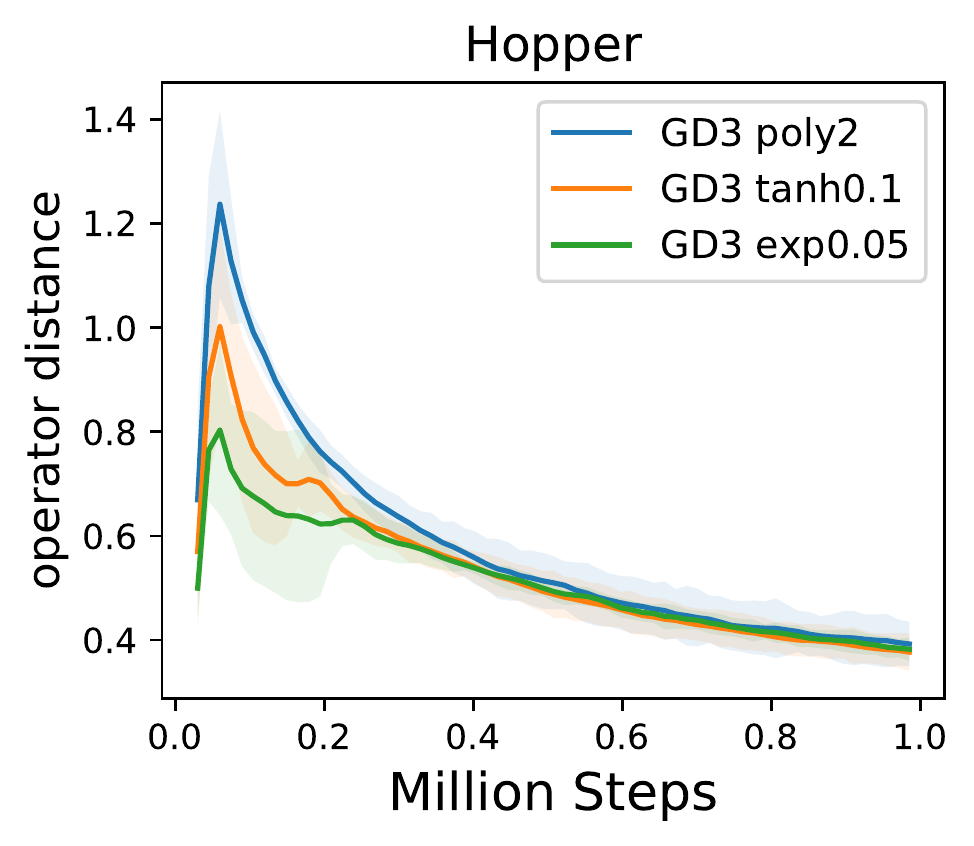}
    }\hspace{0mm}
    \subfigure[Estimation bias]{
    \label{fig:hopperbias}
    \includegraphics[scale=0.45]{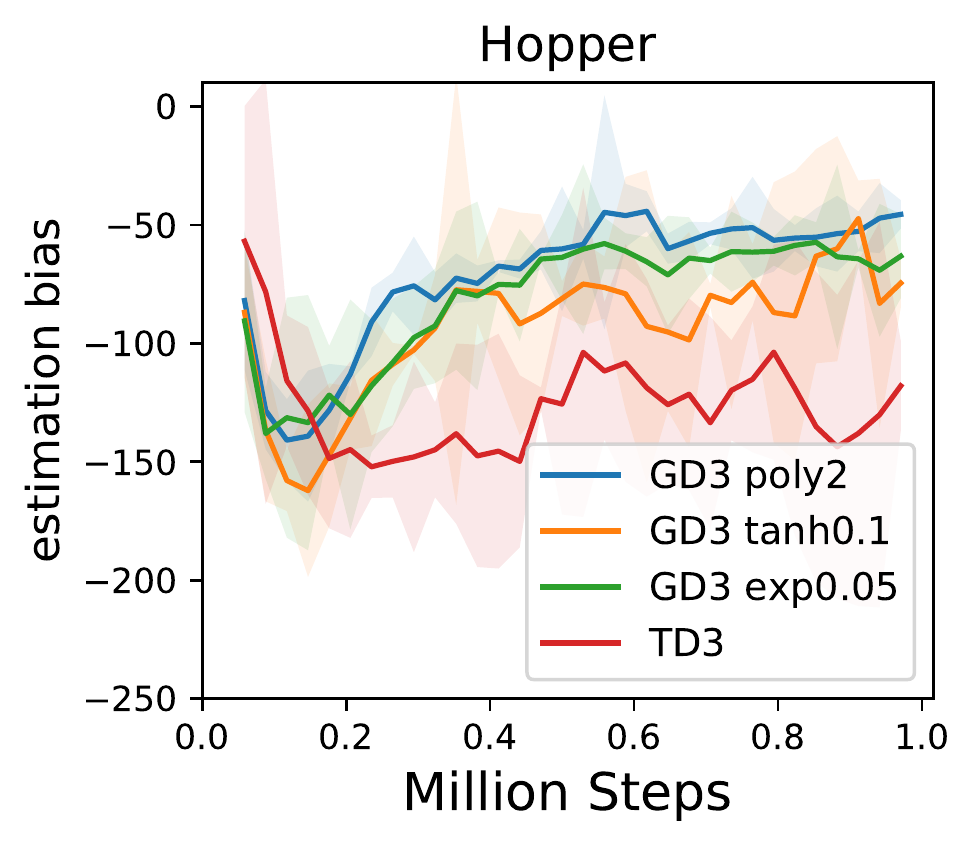}
    }\hspace{0mm}
    \subfigure[Performance]{
    \label{fig:hopperperformance}
    \includegraphics[scale=0.45]{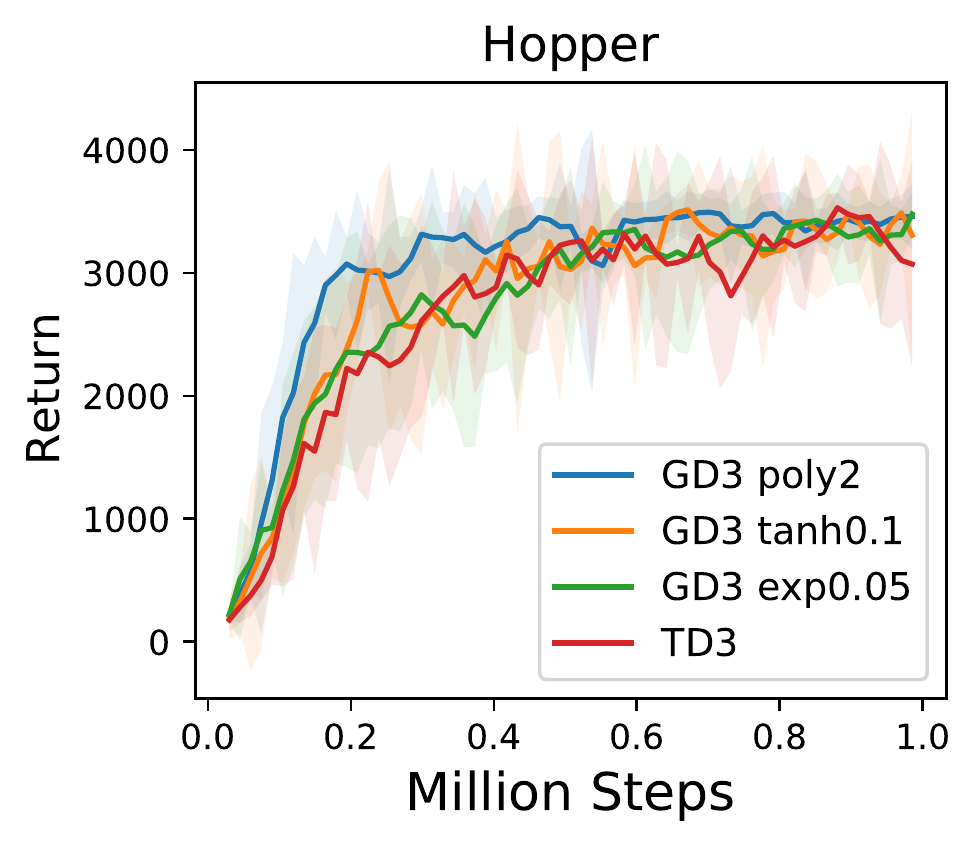}
    }
    \caption{Comparison among exponential activation, polynomial activation and tanh activation: (a) distance between activations and max operator; (b) value estimation bias; (c) overall performance with different activations (5 runs, mean $\pm$ standard deviation).}
    \label{fig:hoppercompare}
\end{figure*}

\noindent {\textbf{Can GD3 alleviate bias?}} Based on Theorem \ref{theo:gd3}, we could get a suitable estimator that could achieve a balance between overestimation bias and underestimation bias with proper activation design. We verify whether simple activations could alleviate the bias by recording the bias induced by GD3 and comparing it with TD3. True values are evaluated by rolling out the present policy and averaging the discounted future rewards over the trajectories which start from the sampled states every $10^5$ steps. Activations help relieve overestimation bias as guaranteed by Corollary~\ref{coro:gd2} but not necessarily underestimation bias, while we can find such function as long as it meets Condition~\ref{cond:ineq}. However, the result in Fig \ref{fig:hopperbias} shows that simple activation functions may induce larger bias than TD3 where all of the activations (polynomial, tanh, exponential) reserve larger bias than TD3. It indicates that simple activations can also empirically benefit underestimation bias alleviation. Fig \ref{fig:hopperperformance} shows that simple activations are enough for good performance where all of them outperform TD3 on Hopper-v2. 

\begin{figure*}
    \centering
    \subfigure[The effect of coefficient $\alpha$.]{
    \label{fig:halfalpha}
    \includegraphics[scale=0.45]{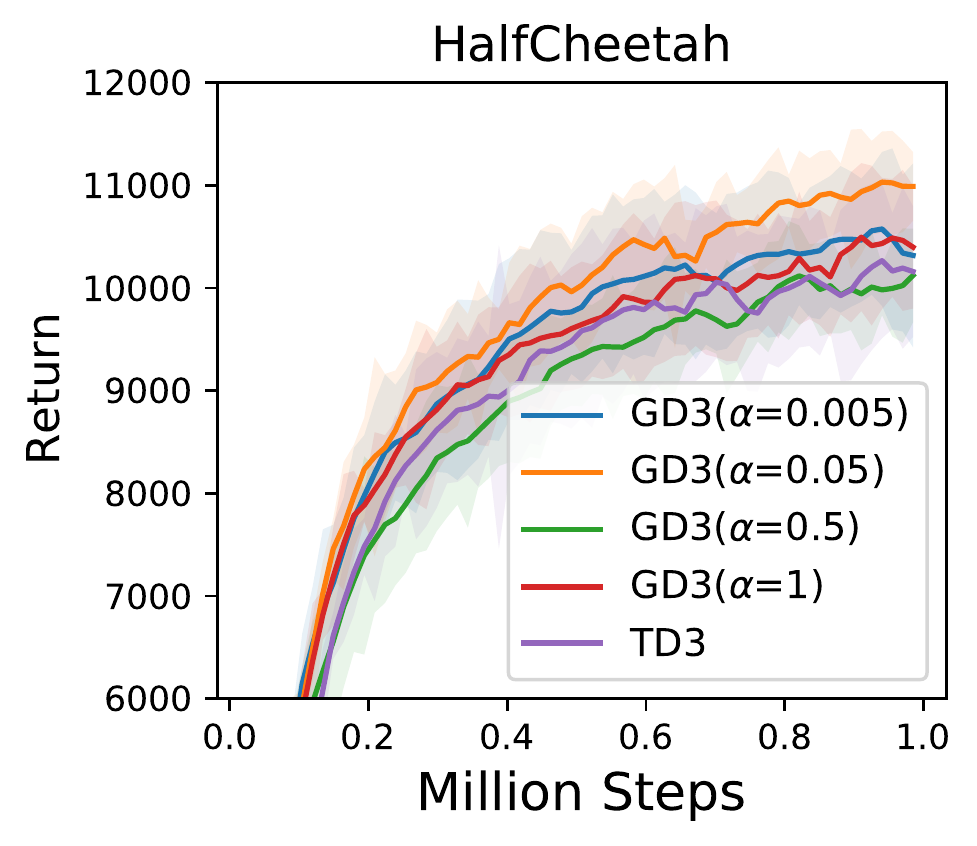}
    }
    \subfigure[The effect of index $k$.]{
    \label{fig:halfk}
    \includegraphics[scale=0.45]{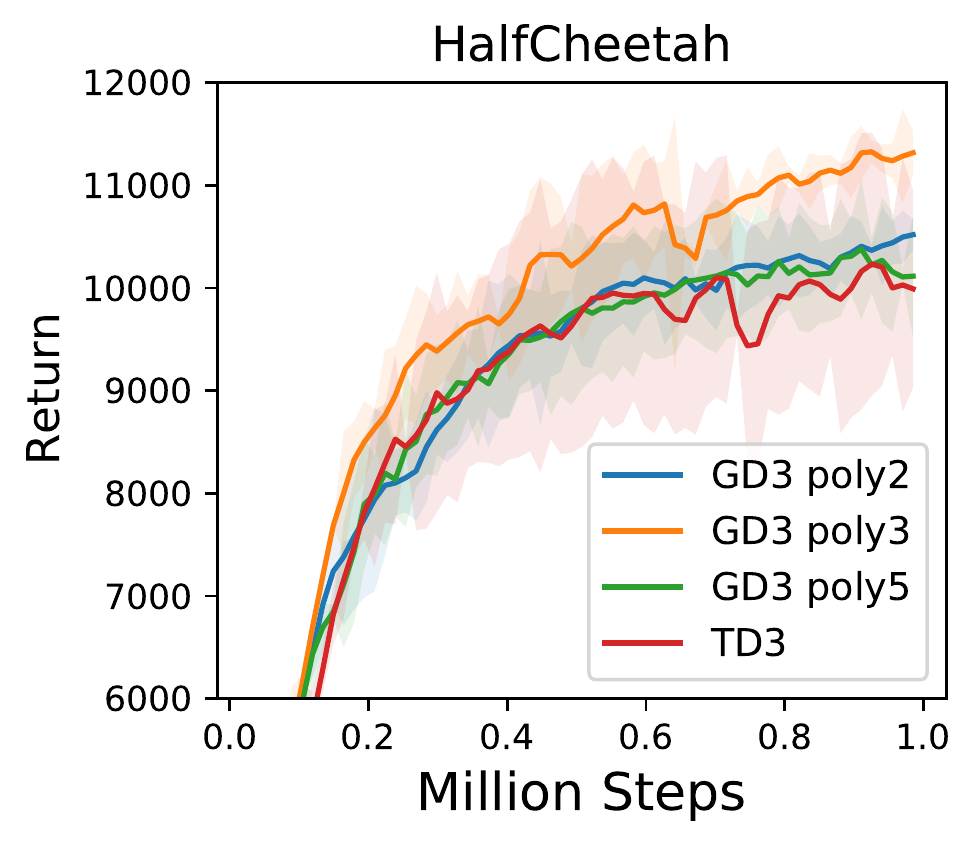}
    }
    \subfigure[The effect of $\beta$ in tanh.]{
    \label{fig:halftanhbeta}
    \includegraphics[scale=0.45]{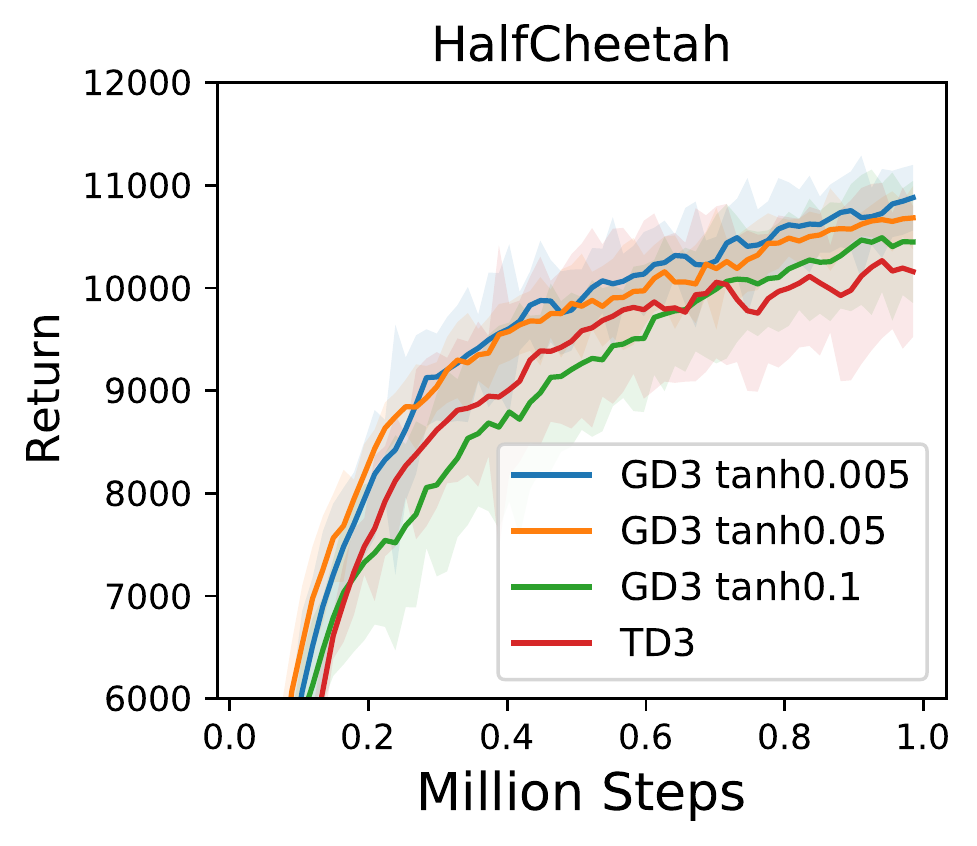}
    }
    \subfigure[The effect of $\beta$ in exp.]{
    \label{fig:halfexpbeta}
    \includegraphics[scale=0.45]{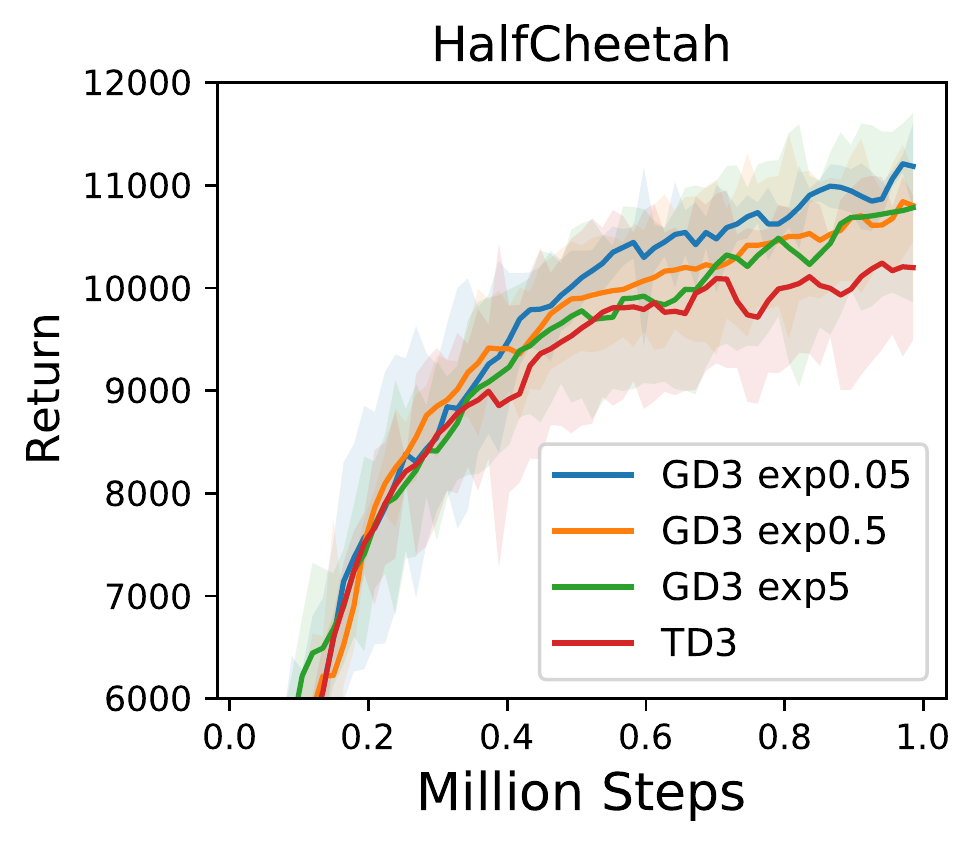}
    }
    \subfigure[The effect of bias $b$.]{
    \label{fig:halfbias}
    \includegraphics[scale=0.45]{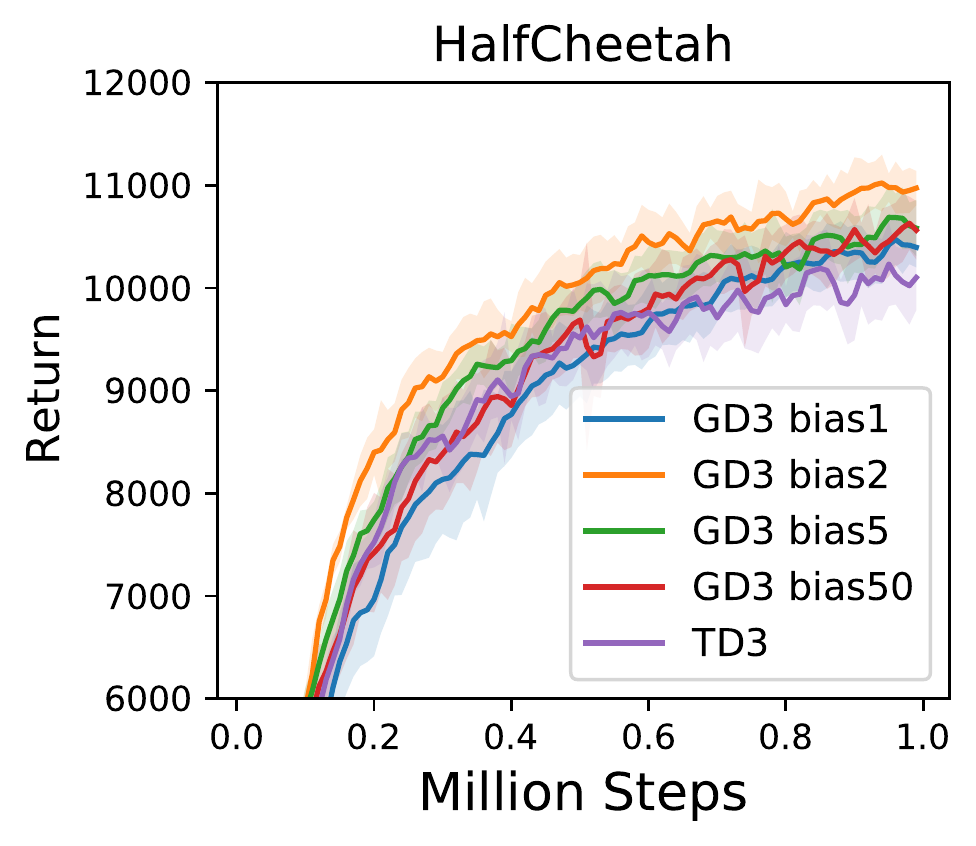}
    }
    \caption{Parameter study and performance comparison. We use polynomial activation with $k=3$ in (a) and (e) (5 runs, mean $\pm$ standard deviation). Note that poly denotes polynomial activation, exp denotes exponential activation, and tanh denotes tanh activation.}
    \label{fig:halfablation}
\end{figure*}

\subsection{How different parameters affect performance?}
\label{sec:ablation}
To better understand how hyperparameters in simple activation functions affect the performance of GD3, we conduct the parameter study in a MuJoCo environment, HalCheetah-v2. For polynomial activation, the coefficient $\alpha$ is vital as small $\alpha$ may lead to a slight underestimation bias while we could achieve an intermediate value that provides a trade-off as shown in Fig \ref{fig:halfalpha}. The index $k$ in polynomial activation and $\beta$ in tanh activation significantly leverage the value estimation result where larger $k$ results in a smaller gap from the max operator. Fig \ref{fig:halfk} and Fig \ref{fig:halfexpbeta} show that too large $k$ in polynomial activation and too large $\beta$ in exponential activation may incur a bad performance. While there does exist a trade-off for $k$ in polynomial activation and $\beta$ in exponential activation. Fig \ref{fig:halftanhbeta} shows that smaller $\beta$ may be better for HalfCheetah-v2, and tanh is not sensitive to $\beta$ as long as it is not too large.  From corollary \ref{coro:bias}, bias term is essential which could result in a softer value estimation while we should not use large bias as it may domain the activation function and lead to even weights (approximate mean operator according to Corollary \ref{coro:zero}). We should use a non-negative bias term to ensure the validity of activation such that larger values would always get larger weights. It turns out that we could achieve an intermediate trade-off for bias as shown in Fig \ref{fig:halfbias}. It is worth noting that almost all GD3 results under different parameters significantly outperform TD3.

\begin{figure*}[!htb]
    \centering
    \includegraphics[scale=0.45]{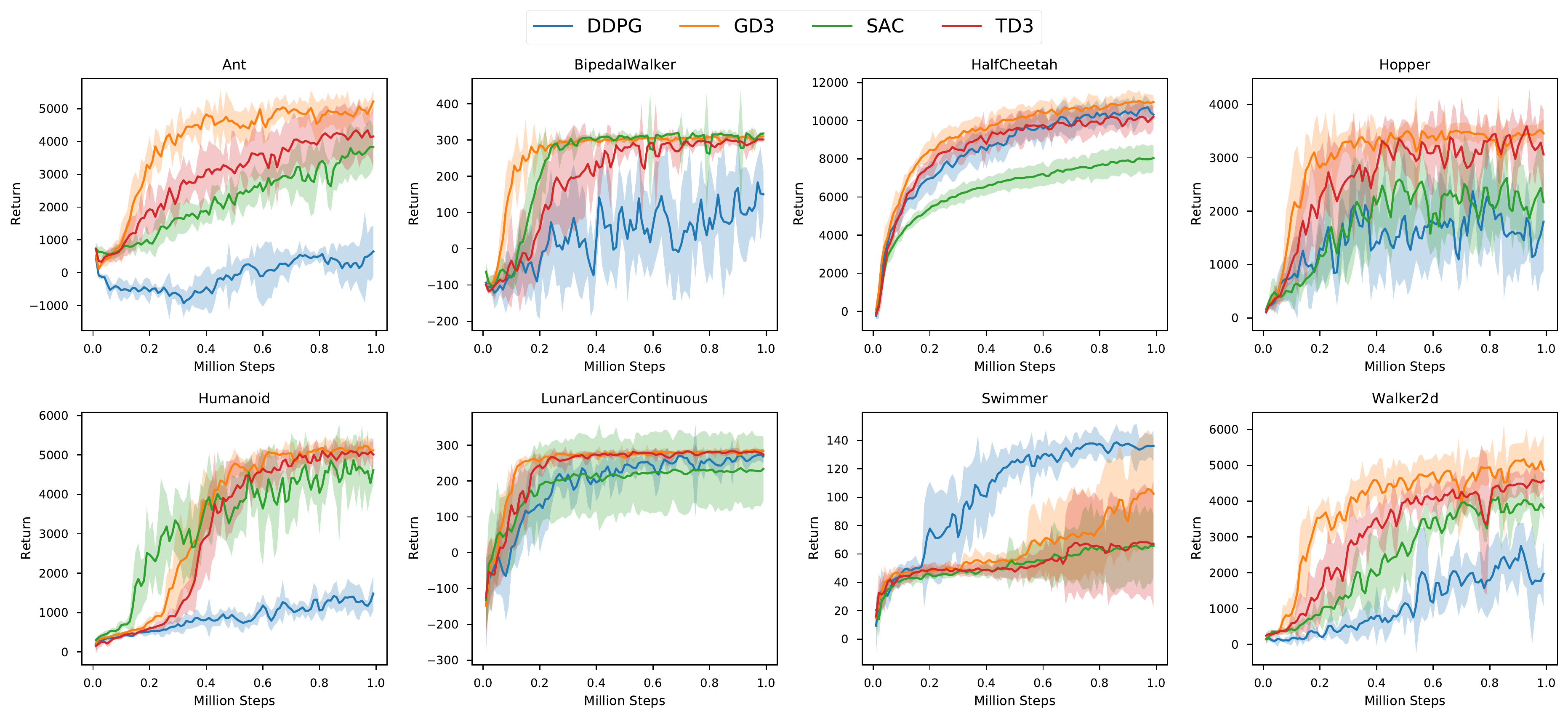}
    \caption{Performance comparison in continuous control tasks. The optimal activation functions are, Ant, HalfCheetah, LunarLancerContinuous: poly3; BipedalWalker, Hopper, Humaniod, Walker2d: poly2; Swimmer: exponential (5 runs, mean $\pm$ standard deviation). Most of the bias terms are set to be 2.}
    \label{fig:gd3result}
\end{figure*}

\subsection{Extensive experiments}
\label{sec:understandgd3}
GD3 is designed for better value estimation where we get a soft value estimation that leans towards the max operator with a little gap by integrating the weighted value function over action space. To further investigate the benefits of simple activations, we conduct additional experiments on other environments. We use common baseline algorithms, DDPG (\cite{lillicrap2015continuous}), TD3 (\cite{fujimoto2018addressing}) and Soft Actor Critic (SAC) (\cite{haarnoja2018soft}), as baselines where DDPG refers to fine-tuned DDPG in (\cite{TD3}) instead of vanilla DDPG implementation (as the vanilla DDPG performs badly). The result could be seen in Fig \ref{fig:gd3result} where we find that GD3 significantly outperforms these common baseline methods with larger return and faster convergence. Notice that optimal activations for different tasks are actually task-specific. While polynomial activation performs well on almost all of the tasks. Various choices of activations make GD3 more flexible and changeable which makes it possible for near-optimal value estimation with reasonable activation functions.

\section{Conclusion}
\label{sec:conclusion}
We propose a general framework for bias alleviation in deep reinforcement learning, namely generalized-activated weighting operator. We show that the distance between generalized-activated weighting operator and max operator can be bounded both theoretically and empirically. We further present \underline{G}eneralized-activated \underline{D}eep \underline{D}ouble \underline{D}eterministic Policy Gradients (GD3) algorithm. GD3 does a correction to the value estimation with function transformation over action space, leveraging any non-decreasing activation functions. We find that there is no need for trivial activation design to specifically meet the bias constraints as optimal activations are task-specific and simple activations are adequate for both higher sample efficiency and good performance. Experimental results on several challenging continuous control tasks show that GD3 with task-specific activations significantly outperforms the common baseline methods. 

For future work, it would be interesting to apply the generalized-activated weighting operator in discrete regime. It would also be interesting to study how to accelerate the learning process of double actors and design a more efficient variant of the GD3 algorithm.








\appendix
\section{Omitted proofs}
\subsection{Proof of Theorem 1}
\begin{proof}
The left hand side could be derived directly by definition:
\begin{equation}
\label{eq:gslfs}
\begin{aligned}
    &GA_g(Q(s,\pi(s));\psi) \\
    &\quad \le \int_{a\in\mathcal{A}}\frac{g(Q(s,a);\psi)}{\int_{a^\prime\in\mathcal{A}}g(Q(s,a^\prime);\psi)da^\prime}\max_{\hat{a}} Q(s,\hat{a})da\\
    &\quad \le \max_a Q(s,a)
\end{aligned}
\end{equation}

As is in (\cite{pan2020softmax}) and (\cite{fujimoto2018addressing}), in order to show the right hand side, we would need the aid of log-sum-g operator, i.e., $$\mathrm{lsg}_\psi(Q(s,\pi(s));\beta) = \frac{1}{\beta}\ln \int_{a\in\mathcal{A}} g(Q(s,a);\psi)da.$$ Note here the activation is a generalized one instead of exponential function, if $\psi=\beta$ and $g(Q(s,a);\psi) = \exp(\psi Q(s,a))$, the lsg operator would become lse operator. Through a non-decreasing function $g(\cdot)$, the estimated $Q$ value is activated with a new density distribution $p_\psi(s,a) = \frac{g(Q(s,a);\psi)}{\int_{a^\prime\in\mathcal{A}} g(Q(s,a^\prime))da^\prime}$. Then from one perspective,
\begin{equation}
\label{eq:gsrhs1}
\begin{aligned}
    &\mathrm{lsg}_\psi(Q(s,\pi(s))) - GA_g(Q(s,\pi(s));\psi) \\
    &= \frac{1}{\beta}\ln \int_{a\in\mathcal{A}} g(Q(s,a);\psi)da - \int_{a\in\mathcal{A}} p_\psi(s,a)Q(s,a)da \\
    &= \frac{1}{\beta}\int_{a\in\mathcal{A}}p_\psi(s,a)\left(\ln \frac{g(Q(s,a);\psi)}{p_\psi(s,a)} \right)da \\ 
    &\qquad \qquad - \int_{a\in\mathcal{A}} p_\psi(s,a)Q(s,a)da \\
    &= \frac{\int_{a\in\mathcal{A}} -p_\psi(s,a)\ln p_\psi(s,a) da}{\beta}  + \int_{a\in\mathcal{A}}p_\psi(s,a)T(Q)da \\
    &\le \frac{\int_{a\in\mathcal{A}1da - 1}}{\beta} + T(Q^*)
\end{aligned}
\end{equation}

where $T(Q(s,a);\psi,\beta) = \frac{1}{\beta}\ln g(Q(s,a);\psi) - Q(s,a)$ and $T(Q^*)$ is the maximum value of this function. The above inequality is induced by the fact that $\forall a$, $-p_\psi(s,a)\ln p_\psi(s,a) \le 1-p_\psi(s,a)$ and $\int_{a\in\mathcal{A}}p_\psi(s,a)da = 1$. Note that $\mathcal{C}(Q,s,\epsilon) \subset \mathcal{A}$, hence,
\begin{equation}
\label{eq:gsrhs2}
    \begin{aligned}
    &\mathrm{lsg}_\psi(Q(s,\pi(s))) \\ &= \frac{\ln \int_{a\in\mathcal{A}} g(Q(s,a);\psi)da}{\beta} \\
    &\ge \frac{\ln \int_{a\in\mathcal{C}(Q,s,\epsilon)} g(\max_{a^\prime}Q(s,a^\prime) - \epsilon;\psi)da}{\beta} \\
    &= \frac{\ln \int_{a\in\mathcal{C}}1da + \ln g(\max_{a^\prime}Q(s,a^\prime) - \epsilon;\psi)}{\beta} \\
    &\ge \frac{\ln F(Q,s,\epsilon) + \ln \exp\left(\beta(\max_{a^\prime}Q(s,a^\prime) - \epsilon) \right)}{\beta} \\
    &= \frac{\ln F(Q,s,\epsilon)}{\beta} + \max_{a^\prime}Q(s,a^\prime) - \epsilon
    \end{aligned}
\end{equation}

We get the inequality that $\ln \int_{a\in\mathcal{A}} g(Q(s,a);\psi)da \ge \ln \int_{a\in\mathcal{C}(Q,s,\epsilon)} g(\max_{a^\prime}Q(s,a^\prime) - \epsilon;\psi)da$ due to the fact that $g(\cdot)$ is a non-decreasing function, hence for $\forall \, Q\in \mathcal{C}(Q,s,\epsilon)$, $g(Q(s,a))\ge g(\max_{a^\prime}Q(s,a^\prime))$. Note that $\beta\in\mathcal{B}$, i.e. $\beta$ is small. By combining Eq.(\ref{eq:gsrhs1}) and Eq. (\ref{eq:gsrhs2}) and denote $M(Q,\epsilon;\psi,\beta) = \epsilon + \frac{\int_{a\in\mathcal{A}}1da - 1 + \ln F(Q,s,\epsilon)}{\beta} + T(Q^*)$, we have the desired conclusion.

\end{proof}

\subsection{Proof of Theorem 3}
\begin{proof}
By using Corollary \ref{coro:gd2} and Lemma \ref{lemma:gd2gd3}, the right hand side can be easily shown. For the left hand side, if the non-decreasing function $g(Q(s,\pi(s);\theta^-);\psi)$ satisfies Condition \ref{cond:inequ}, then the left hand side holds.

\begin{condition}
\label{cond:inequ}
$\int_{a\in\mathcal{A}} g(Q(s,a);\psi) Q(s,a)da \int_{a^\prime\in\mathcal{A}}1da^\prime \ge \int_{a\in\mathcal{A}}g(Q(s,a);\psi)da\int_{a^\prime\in\mathcal{A}}Q(s,a^\prime)da^\prime$
\end{condition}

By definition, we have
\begin{equation}
\label{eq:theo3eq1}
    bias(\mathcal{T}_{TD3}(s))=\mathbb{E}[\hat{Q}_i(s,a)] - \mathbb{E}\left[Q(s,\pi(s;\phi_i^\prime);\theta^\prime) \right]
\end{equation}

where $\theta^\prime$ is the true parameter and the value estimation $\hat{Q}_i(s,a) = \min_{i=1,2} Q_i(s,a;\theta_i^\prime)$.
\begin{equation}
\label{eq:theo3eq2}
    bias(\mathcal{T}_{GD3}(s)) = \mathbb{E}[GA_g(\hat{Q}_i(s,a);\psi)] - \mathbb{E}\left[Q(s,\pi(s;\phi_i^\prime);\theta^\prime) \right]
\end{equation}

Eq. (\ref{eq:theo3eq1}) and Eq. (\ref{eq:theo3eq2}) have the same component. Thus, 
\begin{equation}
    \label{eq:theo3eq3}
    \begin{aligned}
    &bias(\mathcal{T}_{GD3}(s)) - bias(\mathcal{T}_{TD3}(s)) \\ &\quad \quad = \mathbb{E}[GA_g(\hat{Q}_i(s,a);\psi)] - \mathbb{E}[\hat{Q}_i(s,a)]
    \end{aligned}
\end{equation}

If Condition \ref{cond:inequ} holds, we have $\frac{\int_{a\in\mathcal{A}} g(Q(s,a);\psi) Q(s,a)da}{\int_{a\in\mathcal{A}}g(Q(s,a);\psi)da}\ge \frac{\int_{a^\prime\in\mathcal{A}}Q(s,a^\prime)da^\prime}{\int_{a^\prime\in\mathcal{A}}1da^\prime}$ which is exactly $\mathbb{E}[GA_g(\hat{Q}_i(s,a);\psi)]\ge \mathbb{E}[\hat{Q}_i(s,a)]$ and hence the left hand side holds.

It turns out that there exist some non-decreasing functions satisfying such inequality apart from softmax, e.g., the linear activation family. Suppose $g(Q(s,\cdot);\psi) = \psi Q(s,\cdot) + b$, $b\in\mathbb{R}$ and we require $b\ge 0$ here. Then we have the generalized-activated weighting operator: $GA_g(Q(s,\cdot);\psi) = \int_{a\in\mathcal{A}}\frac{(\psi Q(s,a) + b) Q(s,a)}{\int_{a^\prime\in\mathcal{A}} (\psi Q(s,a^\prime) + b) da^\prime}da$. We also have $\psi \ge 0$ as the activation function has to be non-decreasing. By applying Cauchy-Schwarz inequality, it is easy to find
\begin{equation}
    \label{eq:theo2}
    \begin{aligned}
    &\quad \nabla_\psi GA_g(Q(s,\cdot);\psi) \\
    &= \nabla_\psi \int_{a\in\mathcal{A}}\frac{(\psi Q(s,a) + b) Q(s,a)}{\int_{a^\prime\in\mathcal{A}} (\psi Q(s,a^\prime) + b) da^\prime}da \\
    &= \frac{b \left[\int_{a\in\mathcal{A}}Q^2(s,a)da\int_{a^\prime\in\mathcal{A}}1da^\prime - (\int_{a\in\mathcal{A}}Q(s,a)da)^2 \right]}{\left(\int_{a^\prime\in\mathcal{A}}(\psi Q(s,a^\prime) + b)da^\prime\right)^2} \\
    &\ge 0
    \end{aligned}
\end{equation}

Exponential family, i.e., $g(Q(s,\cdot);\psi,k) = k^{\psi Q(s,\cdot)}, k\ge 1, \psi \ge 0,$ also satisfies the inequality, follow similar procedure as the above proof and apply Cauchy-Schwarz inequality, we have
\begin{equation}
    \label{eq:theo2exp}
    \begin{aligned}
    &\quad \nabla_\psi GA_g(Q(s,\cdot);\psi) \\
    &= \nabla_\psi \int_{a\in\mathcal{A}}\frac{k^{\psi Q(s,a)} Q(s,a)}{\int_{a^\prime\in\mathcal{A}} k^{\psi Q(s,a^\prime)} da^\prime}da \\
    &= \frac{\ln k\left(\int_{a\in\mathcal{A}}Q^2(s,a) k^{\psi Q(s,a)}da\int_{a^\prime\in\mathcal{A}}k^{\psi Q(s,a)}da^\prime \right)}{(\int_{a^\prime\in\mathcal{A}}k^{\psi Q(s,a^\prime)}da^\prime)^2} \\ &\quad - \frac{\ln k\left(\int_{a\in\mathcal{A}}Q(s,a)k^{\psi Q(s,a)}da\right)^2}{\left(\int_{a^\prime\in\mathcal{A}}k^{\psi Q(s,a^\prime)}da^\prime\right)^2} \\
    &\ge 0
    \end{aligned}
\end{equation}
\end{proof}

\printcredits

\bibliographystyle{cas-model2-names}

\bibliography{ref.bib}

\begin{thebibliography}{47}
\expandafter\ifx\csname natexlab\endcsname\relax\def\natexlab#1{#1}\fi
\providecommand{\url}[1]{\texttt{#1}}
\providecommand{\href}[2]{#2}
\providecommand{\path}[1]{#1}
\providecommand{\DOIprefix}{doi:}
\providecommand{\ArXivprefix}{arXiv:}
\providecommand{\URLprefix}{URL: }
\providecommand{\Pubmedprefix}{pmid:}
\providecommand{\doi}[1]{\href{http://dx.doi.org/#1}{\path{#1}}}
\providecommand{\Pubmed}[1]{\href{pmid:#1}{\path{#1}}}
\providecommand{\bibinfo}[2]{#2}
\ifx\xfnm\relax \def\xfnm[#1]{\unskip,\space#1}\fi
\bibitem[{Arulkumaran et~al.(2019)Arulkumaran, Cully and
  Togelius}]{Arulkumaran2019AlphaStarAE}
\bibinfo{author}{Arulkumaran, K.}, \bibinfo{author}{Cully, A.},
  \bibinfo{author}{Togelius, J.}, \bibinfo{year}{2019}.
\newblock \bibinfo{title}{Alphastar: an evolutionary computation perspective}.
\newblock \bibinfo{journal}{Proceedings of the Genetic and Evolutionary
  Computation Conference Companion} .
\bibitem[{Barth-Maron et~al.(2018)Barth-Maron, Hoffman, Budden, Dabney, Horgan,
  Tb, Muldal, Heess and Lillicrap}]{barth2018distributed}
\bibinfo{author}{Barth-Maron, G.}, \bibinfo{author}{Hoffman, M.W.},
  \bibinfo{author}{Budden, D.}, \bibinfo{author}{Dabney, W.},
  \bibinfo{author}{Horgan, D.}, \bibinfo{author}{Tb, D.},
  \bibinfo{author}{Muldal, A.}, \bibinfo{author}{Heess, N.},
  \bibinfo{author}{Lillicrap, T.}, \bibinfo{year}{2018}.
\newblock \bibinfo{title}{Distributed distributional deterministic policy
  gradients}.
\newblock \bibinfo{journal}{arXiv preprint arXiv:1804.08617} .
\bibitem[{Bellemare et~al.(2017)Bellemare, Dabney and
  Munos}]{bellemare2017distributional}
\bibinfo{author}{Bellemare, M.G.}, \bibinfo{author}{Dabney, W.},
  \bibinfo{author}{Munos, R.}, \bibinfo{year}{2017}.
\newblock \bibinfo{title}{A distributional perspective on reinforcement
  learning}, in: \bibinfo{booktitle}{International Conference on Machine
  Learning}, pp. \bibinfo{pages}{449--458}.
\bibitem[{Brockman et~al.(2016)Brockman, Cheung, Pettersson, Schneider,
  Schulman, Tang and Zaremba}]{brockman2016openai}
\bibinfo{author}{Brockman, G.}, \bibinfo{author}{Cheung, V.},
  \bibinfo{author}{Pettersson, L.}, \bibinfo{author}{Schneider, J.},
  \bibinfo{author}{Schulman, J.}, \bibinfo{author}{Tang, J.},
  \bibinfo{author}{Zaremba, W.}, \bibinfo{year}{2016}.
\newblock \bibinfo{title}{Openai gym}.
\newblock \bibinfo{journal}{arXiv preprint arXiv:1606.01540} .
\bibitem[{Catto(2011)}]{catto2011box2d}
\bibinfo{author}{Catto, E.}, \bibinfo{year}{2011}.
\newblock \bibinfo{title}{Box2d: A 2d physics engine for games}.
\newblock \bibinfo{journal}{URL: http://www.box2d.org} .
\bibitem[{Cetin and Çeliktutan(2021)}]{Cetin2021LearningPF}
\bibinfo{author}{Cetin, E.}, \bibinfo{author}{Çeliktutan, O.},
  \bibinfo{year}{2021}.
\newblock \bibinfo{title}{Learning pessimism for robust and efficient
  off-policy reinforcement learning}.
\newblock \bibinfo{journal}{ArXiv} \bibinfo{volume}{abs/2110.03375}.
\bibitem[{Chen et~al.(2021)Chen, Jin and Song}]{Chen2021FaulttolerantAT}
\bibinfo{author}{Chen, Q.}, \bibinfo{author}{Jin, Y.}, \bibinfo{author}{Song,
  Y.}, \bibinfo{year}{2021}.
\newblock \bibinfo{title}{Fault-tolerant adaptive tracking control of
  euler-lagrange systems — an echo state network approach driven by
  reinforcement learning}.
\newblock \bibinfo{journal}{Neurocomputing} .
\bibitem[{Ciosek et~al.(2019)Ciosek, Vuong, Loftin and
  Hofmann}]{ciosek2019better}
\bibinfo{author}{Ciosek, K.}, \bibinfo{author}{Vuong, Q.},
  \bibinfo{author}{Loftin, R.}, \bibinfo{author}{Hofmann, K.},
  \bibinfo{year}{2019}.
\newblock \bibinfo{title}{Better exploration with optimistic actor critic}, in:
  \bibinfo{booktitle}{Advances in Neural Information Processing Systems}, pp.
  \bibinfo{pages}{1787--1798}.
\bibitem[{Fujimoto(2018)}]{TD3}
\bibinfo{author}{Fujimoto, S.}, \bibinfo{year}{2018}.
\newblock \bibinfo{title}{Open-source implementation for td3}.
\newblock \bibinfo{howpublished}{\url{https://github.com/sfujim/TD3}}.
\bibitem[{Fujimoto et~al.(2018)Fujimoto, Hoof and
  Meger}]{fujimoto2018addressing}
\bibinfo{author}{Fujimoto, S.}, \bibinfo{author}{Hoof, H.},
  \bibinfo{author}{Meger, D.}, \bibinfo{year}{2018}.
\newblock \bibinfo{title}{Addressing function approximation error in
  actor-critic methods}, in: \bibinfo{booktitle}{International Conference on
  Machine Learning}, pp. \bibinfo{pages}{1587--1596}.
\bibitem[{Gu et~al.(2016)Gu, Lillicrap, Sutskever and
  Levine}]{gu2016continuous}
\bibinfo{author}{Gu, S.}, \bibinfo{author}{Lillicrap, T.},
  \bibinfo{author}{Sutskever, I.}, \bibinfo{author}{Levine, S.},
  \bibinfo{year}{2016}.
\newblock \bibinfo{title}{Continuous deep q-learning with model-based
  acceleration}, in: \bibinfo{booktitle}{International Conference on Machine
  Learning}, pp. \bibinfo{pages}{2829--2838}.
\bibitem[{Haarnoja et~al.(2017)Haarnoja, Tang, Abbeel and
  Levine}]{haarnoja2017reinforcement}
\bibinfo{author}{Haarnoja, T.}, \bibinfo{author}{Tang, H.},
  \bibinfo{author}{Abbeel, P.}, \bibinfo{author}{Levine, S.},
  \bibinfo{year}{2017}.
\newblock \bibinfo{title}{Reinforcement learning with deep energy-based
  policies}, in: \bibinfo{booktitle}{International Conference on Machine
  Learning}, pp. \bibinfo{pages}{1352--1361}.
\bibitem[{Haarnoja et~al.(2018)Haarnoja, Zhou, Abbeel and
  Levine}]{haarnoja2018soft}
\bibinfo{author}{Haarnoja, T.}, \bibinfo{author}{Zhou, A.},
  \bibinfo{author}{Abbeel, P.}, \bibinfo{author}{Levine, S.},
  \bibinfo{year}{2018}.
\newblock \bibinfo{title}{Soft actor-critic: Off-policy maximum entropy deep
  reinforcement learning with a stochastic actor}, in:
  \bibinfo{booktitle}{International Conference on Machine Learning}, pp.
  \bibinfo{pages}{1861--1870}.
\bibitem[{van Hasselt et~al.(2016)van Hasselt, Guez and Silver}]{van2016deep}
\bibinfo{author}{van Hasselt, H.}, \bibinfo{author}{Guez, A.},
  \bibinfo{author}{Silver, D.}, \bibinfo{year}{2016}.
\newblock \bibinfo{title}{Deep reinforcement learning with double q-learning},
  in: \bibinfo{booktitle}{AAAI}.
\bibitem[{He and Hou(2020)}]{He2020ReducingEB}
\bibinfo{author}{He, Q.}, \bibinfo{author}{Hou, X.}, \bibinfo{year}{2020}.
\newblock \bibinfo{title}{Reducing estimation bias via weighted delayed deep
  deterministic policy gradient}.
\newblock \bibinfo{journal}{ArXiv} \bibinfo{volume}{abs/2006.12622}.
\bibitem[{Horgan et~al.(2018)Horgan, Quan, Budden, Barth-Maron, Hessel,
  Van~Hasselt and Silver}]{horgan2018distributed}
\bibinfo{author}{Horgan, D.}, \bibinfo{author}{Quan, J.},
  \bibinfo{author}{Budden, D.}, \bibinfo{author}{Barth-Maron, G.},
  \bibinfo{author}{Hessel, M.}, \bibinfo{author}{Van~Hasselt, H.},
  \bibinfo{author}{Silver, D.}, \bibinfo{year}{2018}.
\newblock \bibinfo{title}{Distributed prioritized experience replay}.
\newblock \bibinfo{journal}{arXiv preprint arXiv:1803.00933} .
\bibitem[{Huang et~al.(2020)Huang, Huang, Hao, Tan, Fan and
  Huang}]{Huang2020AdaptivePS}
\bibinfo{author}{Huang, Q.}, \bibinfo{author}{Huang, R.}, \bibinfo{author}{Hao,
  W.}, \bibinfo{author}{Tan, J.}, \bibinfo{author}{Fan, R.},
  \bibinfo{author}{Huang, Z.}, \bibinfo{year}{2020}.
\newblock \bibinfo{title}{Adaptive power system emergency control using deep
  reinforcement learning}.
\newblock \bibinfo{journal}{IEEE Transactions on Smart Grid}
  \bibinfo{volume}{11}, \bibinfo{pages}{1171--1182}.
\bibitem[{Jiafei et~al.(2021)Jiafei, Xiaoteng, Jiangpeng and
  Xiu}]{lyu2021efficient}
\bibinfo{author}{Jiafei, L.}, \bibinfo{author}{Xiaoteng, M.},
  \bibinfo{author}{Jiangpeng, Y.}, \bibinfo{author}{Xiu, L.},
  \bibinfo{year}{2021}.
\newblock \bibinfo{title}{Efficient continuous control with double actors and
  regularized critics}.
\newblock \bibinfo{journal}{arXiv preprint arXiv:2106.03050} .
\bibitem[{Jiang and Lu(2021)}]{Jiang2021OfflineDM}
\bibinfo{author}{Jiang, J.}, \bibinfo{author}{Lu, Z.}, \bibinfo{year}{2021}.
\newblock \bibinfo{title}{Offline decentralized multi-agent reinforcement
  learning}.
\newblock \bibinfo{journal}{ArXiv} \bibinfo{volume}{abs/2108.01832}.
\bibitem[{Konda and Tsitsiklis(2000)}]{konda2000actor}
\bibinfo{author}{Konda, V.R.}, \bibinfo{author}{Tsitsiklis, J.N.},
  \bibinfo{year}{2000}.
\newblock \bibinfo{title}{Actor-critic algorithms}, in:
  \bibinfo{booktitle}{Advances in neural information processing systems}, pp.
  \bibinfo{pages}{1008--1014}.
\bibitem[{Kuznetsov et~al.(2021)Kuznetsov, Grishin, Tsypin, Ashukha and
  Vetrov}]{Kuznetsov2021AutomatingCO}
\bibinfo{author}{Kuznetsov, A.}, \bibinfo{author}{Grishin, A.},
  \bibinfo{author}{Tsypin, A.}, \bibinfo{author}{Ashukha, A.},
  \bibinfo{author}{Vetrov, D.P.}, \bibinfo{year}{2021}.
\newblock \bibinfo{title}{Automating control of overestimation bias for
  continuous reinforcement learning}.
\newblock \bibinfo{journal}{ArXiv} \bibinfo{volume}{abs/2110.13523}.
\bibitem[{Kuznetsov et~al.(2020)Kuznetsov, Shvechikov, Grishin and
  Vetrov}]{kuznetsov2020control}
\bibinfo{author}{Kuznetsov, A.}, \bibinfo{author}{Shvechikov, P.},
  \bibinfo{author}{Grishin, A.}, \bibinfo{author}{Vetrov, D.},
  \bibinfo{year}{2020}.
\newblock \bibinfo{title}{Controlling overestimation bias with truncated
  mixture of continuous distributional quantile critics}, in:
  \bibinfo{booktitle}{International Conference on Machine Learning}, pp.
  \bibinfo{pages}{5556--5566}.
\bibitem[{Levine et~al.(2020)Levine, Kumar, Tucker and
  Fu}]{Levine2020OfflineRL}
\bibinfo{author}{Levine, S.}, \bibinfo{author}{Kumar, A.},
  \bibinfo{author}{Tucker, G.}, \bibinfo{author}{Fu, J.}, \bibinfo{year}{2020}.
\newblock \bibinfo{title}{Offline reinforcement learning: Tutorial, review, and
  perspectives on open problems}.
\newblock \bibinfo{journal}{ArXiv} \bibinfo{volume}{abs/2005.01643}.
\bibitem[{Lillicrap et~al.(2015)Lillicrap, Hunt, Pritzel, Heess, Erez, Tassa,
  Silver and Wierstra}]{lillicrap2015continuous}
\bibinfo{author}{Lillicrap, T.P.}, \bibinfo{author}{Hunt, J.J.},
  \bibinfo{author}{Pritzel, A.}, \bibinfo{author}{Heess, N.},
  \bibinfo{author}{Erez, T.}, \bibinfo{author}{Tassa, Y.},
  \bibinfo{author}{Silver, D.}, \bibinfo{author}{Wierstra, D.},
  \bibinfo{year}{2015}.
\newblock \bibinfo{title}{Continuous control with deep reinforcement learning}.
\newblock \bibinfo{journal}{arXiv preprint arXiv:1509.02971} .
\bibitem[{Liu et~al.(2020)Liu, Tang, Guo, Li, Ye and He}]{Liu2020TopawareRL}
\bibinfo{author}{Liu, F.}, \bibinfo{author}{Tang, R.}, \bibinfo{author}{Guo,
  H.}, \bibinfo{author}{Li, X.}, \bibinfo{author}{Ye, Y.}, \bibinfo{author}{He,
  X.}, \bibinfo{year}{2020}.
\newblock \bibinfo{title}{Top-aware reinforcement learning based
  recommendation}.
\newblock \bibinfo{journal}{Neurocomputing} \bibinfo{volume}{417},
  \bibinfo{pages}{255--269}.
\bibitem[{Liu et~al.(2021)Liu, Cao, Wang, Chen and Liu}]{Liu2021SelfplayRL}
\bibinfo{author}{Liu, S.}, \bibinfo{author}{Cao, J.}, \bibinfo{author}{Wang,
  Y.}, \bibinfo{author}{Chen, W.}, \bibinfo{author}{Liu, Y.},
  \bibinfo{year}{2021}.
\newblock \bibinfo{title}{Self-play reinforcement learning with comprehensive
  critic in computer games}.
\newblock \bibinfo{journal}{Neurocomputing} \bibinfo{volume}{449},
  \bibinfo{pages}{207--213}.
\bibitem[{Ma et~al.(2021)Ma, Zhang, Liu, Ji and Gao}]{Ma2021APM}
\bibinfo{author}{Ma, C.}, \bibinfo{author}{Zhang, J.}, \bibinfo{author}{Liu,
  J.}, \bibinfo{author}{Ji, L.}, \bibinfo{author}{Gao, F.},
  \bibinfo{year}{2021}.
\newblock \bibinfo{title}{A parallel multi-module deep reinforcement learning
  algorithm for stock trading}.
\newblock \bibinfo{journal}{Neurocomputing} \bibinfo{volume}{449},
  \bibinfo{pages}{290--302}.
\bibitem[{Ma et~al.(2020)Ma, Xia, Zhou, Yang and Zhao}]{Ma2020DSACDS}
\bibinfo{author}{Ma, X.}, \bibinfo{author}{Xia, L.}, \bibinfo{author}{Zhou,
  Z.}, \bibinfo{author}{Yang, J.}, \bibinfo{author}{Zhao, Q.},
  \bibinfo{year}{2020}.
\newblock \bibinfo{title}{Dsac: Distributional soft actor critic for
  risk-sensitive reinforcement learning}.
\newblock \bibinfo{journal}{arXiv: Learning} .
\bibitem[{Meng et~al.(2020)Meng, Gorbet and Kuli{\'c}}]{meng2020effect}
\bibinfo{author}{Meng, L.}, \bibinfo{author}{Gorbet, R.},
  \bibinfo{author}{Kuli{\'c}, D.}, \bibinfo{year}{2020}.
\newblock \bibinfo{title}{The effect of multi-step methods on overestimation in
  deep reinforcement learning}.
\newblock \bibinfo{journal}{arXiv preprint arXiv:2006.12692} .
\bibitem[{Mirhoseini et~al.(2020)Mirhoseini, Goldie, Yazgan, Jiang, Songhori,
  Wang, Lee, Johnson, Pathak, Bae, Nazi, Pak, Tong, Srinivasa, Hang, Tuncer,
  Babu, Le, Laudon, Ho, Carpenter and Dean}]{Mirhoseini2020ChipPW}
\bibinfo{author}{Mirhoseini, A.}, \bibinfo{author}{Goldie, A.},
  \bibinfo{author}{Yazgan, M.}, \bibinfo{author}{Jiang, J.W.},
  \bibinfo{author}{Songhori, E.M.}, \bibinfo{author}{Wang, S.},
  \bibinfo{author}{Lee, Y.J.}, \bibinfo{author}{Johnson, E.},
  \bibinfo{author}{Pathak, O.}, \bibinfo{author}{Bae, S.},
  \bibinfo{author}{Nazi, A.}, \bibinfo{author}{Pak, J.}, \bibinfo{author}{Tong,
  A.}, \bibinfo{author}{Srinivasa, K.}, \bibinfo{author}{Hang, W.},
  \bibinfo{author}{Tuncer, E.}, \bibinfo{author}{Babu, A.},
  \bibinfo{author}{Le, Q.V.}, \bibinfo{author}{Laudon, J.},
  \bibinfo{author}{Ho, R.}, \bibinfo{author}{Carpenter, R.},
  \bibinfo{author}{Dean, J.}, \bibinfo{year}{2020}.
\newblock \bibinfo{title}{Chip placement with deep reinforcement learning}.
\newblock \bibinfo{journal}{ArXiv} \bibinfo{volume}{abs/2004.10746}.
\bibitem[{Mnih et~al.(2016)Mnih, Badia, Mirza, Graves, Lillicrap, Harley,
  Silver and Kavukcuoglu}]{mnih2016asynchronous}
\bibinfo{author}{Mnih, V.}, \bibinfo{author}{Badia, A.P.},
  \bibinfo{author}{Mirza, M.}, \bibinfo{author}{Graves, A.},
  \bibinfo{author}{Lillicrap, T.}, \bibinfo{author}{Harley, T.},
  \bibinfo{author}{Silver, D.}, \bibinfo{author}{Kavukcuoglu, K.},
  \bibinfo{year}{2016}.
\newblock \bibinfo{title}{Asynchronous methods for deep reinforcement
  learning}, in: \bibinfo{booktitle}{International conference on machine
  learning}, pp. \bibinfo{pages}{1928--1937}.
\bibitem[{Mnih et~al.(2015)Mnih, Kavukcuoglu, Silver, Rusu, Veness, Bellemare,
  Graves, Riedmiller, Fidjeland, Ostrovski et~al.}]{mnih2015human}
\bibinfo{author}{Mnih, V.}, \bibinfo{author}{Kavukcuoglu, K.},
  \bibinfo{author}{Silver, D.}, \bibinfo{author}{Rusu, A.A.},
  \bibinfo{author}{Veness, J.}, \bibinfo{author}{Bellemare, M.G.},
  \bibinfo{author}{Graves, A.}, \bibinfo{author}{Riedmiller, M.},
  \bibinfo{author}{Fidjeland, A.K.}, \bibinfo{author}{Ostrovski, G.}, et~al.,
  \bibinfo{year}{2015}.
\newblock \bibinfo{title}{Human-level control through deep reinforcement
  learning}.
\newblock \bibinfo{journal}{nature} \bibinfo{volume}{518},
  \bibinfo{pages}{529--533}.
\bibitem[{Moskovitz et~al.(2021)Moskovitz, Parker-Holder, Pacchiano, Arbel and
  Jordan}]{Moskovitz2021tactical}
\bibinfo{author}{Moskovitz, T.}, \bibinfo{author}{Parker-Holder, J.},
  \bibinfo{author}{Pacchiano, A.}, \bibinfo{author}{Arbel, M.},
  \bibinfo{author}{Jordan, M.}, \bibinfo{year}{2021}.
\newblock \bibinfo{title}{Tactical optimism and pessimism for deep
  reinforcement learning}, in: \bibinfo{booktitle}{Advances in Neural
  Information Processing Systems}.
\bibitem[{P{\'a}los and Husz{\'a}k(2020)}]{Plos2020ComparisonOQ}
\bibinfo{author}{P{\'a}los, P.}, \bibinfo{author}{Husz{\'a}k, {\'A}.},
  \bibinfo{year}{2020}.
\newblock \bibinfo{title}{Comparison of q-learning based traffic light control
  methods and objective functions}.
\newblock \bibinfo{journal}{2020 International Conference on Software,
  Telecommunications and Computer Networks (SoftCOM)} , \bibinfo{pages}{1--6}.
\bibitem[{Pan et~al.(2020)Pan, Cai and Huang}]{pan2020softmax}
\bibinfo{author}{Pan, L.}, \bibinfo{author}{Cai, Q.}, \bibinfo{author}{Huang,
  L.}, \bibinfo{year}{2020}.
\newblock \bibinfo{title}{Softmax deep double deterministic policy gradients}.
\newblock \bibinfo{journal}{Advances in Neural Information Processing Systems}
  \bibinfo{volume}{33}.
\bibitem[{Pei et~al.(2019)Pei, Yang, Cui, Lin, Sun, Jiang, Ou and
  Zhang}]{Pei2019ValueawareRB}
\bibinfo{author}{Pei, C.}, \bibinfo{author}{Yang, X.}, \bibinfo{author}{Cui,
  Q.}, \bibinfo{author}{Lin, X.}, \bibinfo{author}{Sun, F.},
  \bibinfo{author}{Jiang, P.}, \bibinfo{author}{Ou, W.},
  \bibinfo{author}{Zhang, Y.}, \bibinfo{year}{2019}.
\newblock \bibinfo{title}{Value-aware recommendation based on reinforcement
  profit maximization}.
\newblock \bibinfo{journal}{The World Wide Web Conference} .
\bibitem[{Popov et~al.(2017)Popov, Heess, Lillicrap, Hafner, Barth-Maron,
  Vecerik, Lampe, Tassa, Erez and Riedmiller}]{popov2017data}
\bibinfo{author}{Popov, I.}, \bibinfo{author}{Heess, N.},
  \bibinfo{author}{Lillicrap, T.}, \bibinfo{author}{Hafner, R.},
  \bibinfo{author}{Barth-Maron, G.}, \bibinfo{author}{Vecerik, M.},
  \bibinfo{author}{Lampe, T.}, \bibinfo{author}{Tassa, Y.},
  \bibinfo{author}{Erez, T.}, \bibinfo{author}{Riedmiller, M.},
  \bibinfo{year}{2017}.
\newblock \bibinfo{title}{Data-efficient deep reinforcement learning for
  dexterous manipulation}.
\newblock \bibinfo{journal}{arXiv preprint arXiv:1704.03073} .
\bibitem[{Silver et~al.(2014)Silver, Lever, Heess, Degris, Wierstra and
  Riedmiller}]{silver2014deterministic}
\bibinfo{author}{Silver, D.}, \bibinfo{author}{Lever, G.},
  \bibinfo{author}{Heess, N.}, \bibinfo{author}{Degris, T.},
  \bibinfo{author}{Wierstra, D.}, \bibinfo{author}{Riedmiller, M.},
  \bibinfo{year}{2014}.
\newblock \bibinfo{title}{Deterministic policy gradient algorithms}, in:
  \bibinfo{booktitle}{International Conference on Machine Learning}, pp.
  \bibinfo{pages}{387--395}.
\bibitem[{Silver et~al.(2015)Silver, Schrittwieser, Simonyan, Antonoglou,
  Huang, Guez, Hubert and et~al.}]{mastering2015}
\bibinfo{author}{Silver, D.}, \bibinfo{author}{Schrittwieser, J.},
  \bibinfo{author}{Simonyan, K.}, \bibinfo{author}{Antonoglou, I.},
  \bibinfo{author}{Huang, A.}, \bibinfo{author}{Guez, A.},
  \bibinfo{author}{Hubert, T.}, \bibinfo{author}{et~al.}, \bibinfo{year}{2015}.
\newblock \bibinfo{title}{Mastering the game of go without human knowledge}.
\newblock \bibinfo{journal}{nature} \bibinfo{volume}{550},
  \bibinfo{pages}{354--359}.
\bibitem[{Singh et~al.(1994)Singh, Jaakkola and Jordan}]{singh1994convergence}
\bibinfo{author}{Singh, S.P.}, \bibinfo{author}{Jaakkola, T.},
  \bibinfo{author}{Jordan, M.I.}, \bibinfo{year}{1994}.
\newblock \bibinfo{title}{Convergence of stochastic iterative dynamic
  programming algorithms}, in: \bibinfo{booktitle}{Advances in Neural
  Information Processing Systems}, pp. \bibinfo{pages}{703--710}.
\bibitem[{Sutton et~al.(2000)Sutton, McAllester, Singh and
  Mansour}]{sutton2000policy}
\bibinfo{author}{Sutton, R.S.}, \bibinfo{author}{McAllester, D.A.},
  \bibinfo{author}{Singh, S.P.}, \bibinfo{author}{Mansour, Y.},
  \bibinfo{year}{2000}.
\newblock \bibinfo{title}{Policy gradient methods for reinforcement learning
  with function approximation}, in: \bibinfo{booktitle}{Advances in neural
  information processing systems}, pp. \bibinfo{pages}{1057--1063}.
\bibitem[{Todorov et~al.(2012)Todorov, Erez and Tassa}]{todorov2012mujoco}
\bibinfo{author}{Todorov, E.}, \bibinfo{author}{Erez, T.},
  \bibinfo{author}{Tassa, Y.}, \bibinfo{year}{2012}.
\newblock \bibinfo{title}{Mujoco: A physics engine for model-based control},
  in: \bibinfo{booktitle}{2012 IEEE/RSJ International Conference on Intelligent
  Robots and Systems}, \bibinfo{organization}{IEEE}. pp.
  \bibinfo{pages}{5026--5033}.
\bibitem[{Watkins and Cornish~Hellaby(1989)}]{watkins1989learning}
\bibinfo{author}{Watkins, C.}, \bibinfo{author}{Cornish~Hellaby, J.},
  \bibinfo{year}{1989}.
\newblock \bibinfo{title}{Learning from delayed rewards} .
\bibitem[{Yang et~al.(2021)Yang, Ma, Li, Zheng, Zhang, Huang, Yang and
  Zhao}]{Yang2021BelieveWY}
\bibinfo{author}{Yang, Y.}, \bibinfo{author}{Ma, X.}, \bibinfo{author}{Li, C.},
  \bibinfo{author}{Zheng, Z.}, \bibinfo{author}{Zhang, Q.},
  \bibinfo{author}{Huang, G.}, \bibinfo{author}{Yang, J.},
  \bibinfo{author}{Zhao, Q.}, \bibinfo{year}{2021}.
\newblock \bibinfo{title}{Believe what you see: Implicit constraint approach
  for offline multi-agent reinforcement learning}.
\newblock \bibinfo{journal}{ArXiv} \bibinfo{volume}{abs/2106.03400}.
\bibitem[{Ye et~al.(2020a)Ye, Chen, Zhang, Chen, Yuan, Liu, Chen, Liu, Qiu, Yu,
  Yin, Shi, Wang, Shi, Fu, Yang, Huang and Liu}]{Ye2020TowardsPF}
\bibinfo{author}{Ye, D.}, \bibinfo{author}{Chen, G.}, \bibinfo{author}{Zhang,
  W.}, \bibinfo{author}{Chen, S.}, \bibinfo{author}{Yuan, B.},
  \bibinfo{author}{Liu, B.}, \bibinfo{author}{Chen, J.}, \bibinfo{author}{Liu,
  Z.}, \bibinfo{author}{Qiu, F.}, \bibinfo{author}{Yu, H.},
  \bibinfo{author}{Yin, Y.}, \bibinfo{author}{Shi, B.}, \bibinfo{author}{Wang,
  L.}, \bibinfo{author}{Shi, T.}, \bibinfo{author}{Fu, Q.},
  \bibinfo{author}{Yang, W.}, \bibinfo{author}{Huang, L.},
  \bibinfo{author}{Liu, W.}, \bibinfo{year}{2020}a.
\newblock \bibinfo{title}{Towards playing full moba games with deep
  reinforcement learning}.
\newblock \bibinfo{journal}{ArXiv} \bibinfo{volume}{abs/2011.12692}.
\bibitem[{Ye et~al.(2020b)Ye, Liu, Sun, Shi, Zhao, Wu, Yu, Yang, Wu, Guo, Chen,
  Yin, Zhang, Shi, Wang, Fu, Yang and Huang}]{Ye2020MasteringCC}
\bibinfo{author}{Ye, D.}, \bibinfo{author}{Liu, Z.}, \bibinfo{author}{Sun, M.},
  \bibinfo{author}{Shi, B.}, \bibinfo{author}{Zhao, P.}, \bibinfo{author}{Wu,
  H.}, \bibinfo{author}{Yu, H.}, \bibinfo{author}{Yang, S.},
  \bibinfo{author}{Wu, X.}, \bibinfo{author}{Guo, Q.}, \bibinfo{author}{Chen,
  Q.}, \bibinfo{author}{Yin, Y.}, \bibinfo{author}{Zhang, H.},
  \bibinfo{author}{Shi, T.}, \bibinfo{author}{Wang, L.}, \bibinfo{author}{Fu,
  Q.}, \bibinfo{author}{Yang, W.}, \bibinfo{author}{Huang, L.},
  \bibinfo{year}{2020}b.
\newblock \bibinfo{title}{Mastering complex control in moba games with deep
  reinforcement learning}.
\newblock \bibinfo{journal}{ArXiv} \bibinfo{volume}{abs/1912.09729}.
\bibitem[{Zhang et~al.(2017)Zhang, Pan and Kochenderfer}]{zhang2017weighted}
\bibinfo{author}{Zhang, Z.}, \bibinfo{author}{Pan, Z.},
  \bibinfo{author}{Kochenderfer, M.J.}, \bibinfo{year}{2017}.
\newblock \bibinfo{title}{Weighted double q-learning.}, in:
  \bibinfo{booktitle}{IJCAI}, pp. \bibinfo{pages}{3455--3461}.

\end{thebibliography}

\bio{JiafeiLyu.jpg}
\textbf{Jiafei Lyu}.
{Jiafei Lyu} received his B.S. degree in Engineering Physics from Tsinghua University, Beijing, China, in 2020. He is currently pursuing a Ph.D. degree in Tsinghua Shenzhen International Graduate School. His research interest includes deep reinforcement learning, big data analysis, machine learning, and artificial intelligence.
\endbio

\bio{yuyang.jpg}
\textbf{Yu Yang}.
{Yu Yang} received his B.S. degree in Beijing Jiaotong University, Beijing, China, in 2020. He is currently a master student at Tsinghua Shenzhen International Graduate School. His research interest includes deep reinforcement learning, especially multi-agent deep reinforcement learning, deep learning, and artificial intelligence.
\endbio

\bio[width=2.5cm]{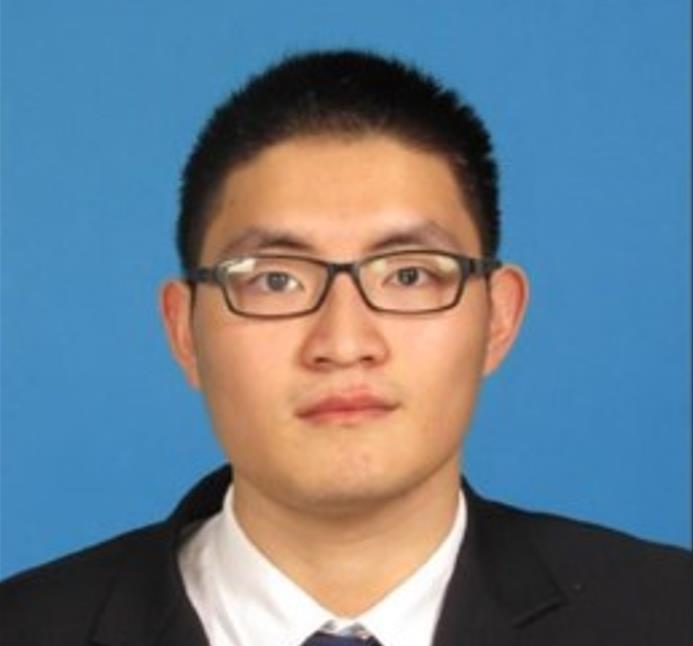}
\textbf{Jiangpeng Yan}.
{Jiangpeng Yan} received his B.S. degree in the Department of Automation from Tsinghua University, Beijing, China, in 2017. He is currently a Ph.D. Student in Department of Automation, Tsinghua University, China. His research interests include medical image analysis, reinforcement learning, and artificial intelligence.
\endbio

\bio{xiuli.jpg}
\textbf{Xiu Li}.
{Xiu Li}(Member, IEEE) received her Ph.D. degree in computer integrated manufacturing in 2000. Since then, she has been working with Tsinghua University. Her research interests are in the areas of data mining, deep learning, computer vision, reinforcement learning, and image processing.
\endbio


\end{document}